\documentclass[usletter]{amsart} 

\usepackage{amsmath,amsthm,amssymb,amsfonts,mathrsfs,color,hyperref, mathtools,crop, graphicx, enumitem, todonotes}
\usepackage[]{geometry}

\theoremstyle{plain}
\begingroup
\newtheorem{theorem}{Theorem}[section]
\newtheorem*{theorem*}{Theorem}
\newtheorem*{"theorem"}{``Theorem''}
\newtheorem{corollary}[theorem]{Corollary}
\newtheorem{proposition}[theorem]{Proposition}

\newtheorem{lemma}[theorem]{Lemma}
\endgroup

\theoremstyle{definition}
\begingroup

\endgroup

\theoremstyle{remark}
\begingroup
\newtheorem{remark}[theorem]{Remark}
\newtheorem{example}[theorem]{Example}
\endgroup 

\numberwithin{equation}{section}
\newcommand{\N}{\mathbb N}

\newcommand{\R}{\mathbb R} 

\renewcommand{\P}{{\mathbb P}}

\newcommand{\dist}{{\rm dist}}
\newcommand{\diam}{{\rm diam}}

\renewcommand{\H}{{\mathcal H}}

\newcommand{\M}{{\mathcal M}}

\renewcommand{\L}{{\mathcal L}}

\newcommand{\F}{{\mathcal F}}
\newcommand{\G}{{\mathcal G}}

\newcommand{\LRa} {\Leftrightarrow}
\newcommand{\Ra} {\Rightarrow}

\renewcommand{\d}{\mathrm{d}}

\newcommand{\dx}{\,\mathrm{d}x}

\newcommand{\dz}{\,\mathrm{d}z}
\newcommand{\ds}{\,\mathrm{d}s}

\newcommand{\dr}{\,\mathrm{d}r}

\newcommand{\eps}{\varepsilon}
\newcommand{\average}{{\mathchoice {\kern1ex\vcenter{\hrule height.4pt
width 6pt depth0pt} \kern-9.7pt} {\kern1ex\vcenter{\hrule
height.4pt width 4.3pt depth0pt} \kern-7pt} {} {} }}
\newcommand{\avint}{\average\int}

\allowdisplaybreaks

 \makeatletter
\@namedef{subjclassname@2020}{2020 Mathematics Subject Classification}
\makeatother
\newcommand\showlabel{\addtocounter{equation}{1}\tag{\theequation}}

\DeclareMathOperator*{\argmin}{argmin} 
\DeclareMathOperator*{\argmax}{argmax} 

\newcommand{\B}{\mathcal{B}}
\newcommand{\meas}{\mathrm{meas}}
 
\begin{document}

\title[Bump functions and shallow ReLU networks]{Optimal bump functions for shallow ReLU networks\\ {\Small Weight decay, depth separation and the curse of dimensionality}}

\author{Stephan Wojtowytsch}
\address{Stephan Wojtowytsch\\
Department of Mathematics\\
Texas A\&M University\\
155 Ireland Street\\
College Station, TX 77840
}
\email{swoj@tamu.edu}

\date{\today}


\begin{abstract}
In this note, we study how neural networks with a single hidden layer and ReLU activation interpolate data drawn from a radially symmetric distribution with target labels 1 at the origin and 0 outside the unit ball, if no labels are known inside the unit ball. With weight decay regularization and in the infinite neuron, infinite data limit, we prove that a unique radially symmetric minimizer exists, whose weight decay regularizer and Lipschitz constant grow as $d$ and $\sqrt{d}$ respectively. 

We furthermore show that the weight decay regularizer grows exponentially in $d$ if the label $1$ is imposed on a ball of radius $\eps$ rather than just at the origin. By comparison, a neural networks with two hidden layers can approximate the target function without encountering the curse of dimensionality.
\end{abstract}

\maketitle

\section{Introduction}

Neural networks have revolutionized fields from computer vision \cite{krizhevsky2012imagenet} to natural language processing \cite{vaswani2017attention}. They are the driving force behind AIs which play strategy games at superhuman levels of proficiency \cite{silver2018general,silver2016mastering,silver2017mastering}, facilitated major advances in scientific problems such as protein folding \cite{tunyasuvunakool2021highly,jumper2021highly}, and have been used for computer-assisted proofs in applied mathematics \cite{wang2022self}. While empirical evidence indicates that they often generalize well to previously unseen data when trained appropriately, there is little rigorous understanding of how neural networks interpolate a function between known data points.

In this article, we provide insight in the simple setting of infinitely wide ReLU networks with a single hidden layer and data which are drawn from a radially symmetric distribution on a Euclidean space $\R^d$. The target function $f^*$ satisfies $f^*(0) = 1$ and $f^*(x) = 0$ for $|x|\geq 1$, where $|\cdot|$ denotes the Euclidean norm on $\R^d$. We consider a loss functional composed of an $\ell^2$-error and a weight decay regularizer. Despite the fact that neural networks with a single hidden layer cannot represent compactly supported target functions exactly \cite{lu2021note}, there are such functions which can be approximated efficiently even in high dimension. Here, we construct optimal infinitely wide networks, and show that the weight decay regularizer grows only linearly in the dimension $d$ of the data space, improving on the quadratic upper bound established by \cite{ongie2019function}.

While highly idealized, this setting allows us to study several important aspects of neural network models:

\begin{enumerate}
\item {\bf Learning symmetries.} The target function has two important symmetries:
\begin{itemize}

\item $f^*$ is radially symmetric. While it is impossible to fit this symmetry exactly by finite networks, it can be attained asymptotically for highly overparametrized networks. More precisely, one could ask whether regularized risk minimization leads to symmetry learning. While we show that a unique radially symmetric solution exists, it remains open whether other solutions exist which do not exhibit radial symmetry. 

\item $0\leq f^*\leq 1$ almost everywhere with respect to the data distribution. Unlike linear models, which necessarily output negative data even if all training data points are positive, neural networks have the capacity to respect this constraint. We show that risk minimization asymptotically enforces the bound {\em everywhere} on the data space, at least for the unique radially symmetric solution.

\end{itemize} 

\item {\bf Fitting random or perturbed data.} It is known that overparametrized neural networks can fit random data, but due to the great generality of the result, the network weights may be prohibitively large for given data. Assuming that labels are generated by a function which can be approximated well by a neural network, compactly supported bump functions can be used to obtain an upper bound on the magnitude needed for specific labels or the increase necessary in the weight decay regularizer if the labels are perturbed.

\item {\bf Depth separation and curse of dimensionality.} We prove two complimentary results:
\begin{itemize}
\item In dimension $d$, there exists an infinitely wide ReLU network with one hidden layer $f_d^*$ with weight decay regularizer $\sim d$ such that $f_d^*(0) = 1$ and $f_d^*(x) = 0$ if $|x|\geq 1$.

\item If $f_{d,\eps}^*$ is an infinitely wide ReLU network with one hidden layer such that $f_d^*(x) = 1$ for $|x|\leq \eps$ and $f_d^*(x) = 0$ if $|x|\geq 1$, then the weight decay regularizer of $f_d^*$ grows at least exponentially as $\eps^2d^{1/2}(1-\eps^2)^{-\frac{d+1}2}$ in the dimension $d$ of the data space.
\end{itemize}

The curse of dimensionality can be avoided in the second situation by using a neural network with two hidden layers, for which the weight decay regularizer only grows as $\sim d^{1/3}(1-\eps)^{-1}$.

\item {\bf Effect of regularization.} Weight decay regularization is often taken as a proxy for controlling the Lipschitz constant of a neural network, as it can be computed more easily. In this highly symmetric setting, we can compare two optimal solutions:

\begin{enumerate}
\item The data is fitted optimally by the function $\hat f_d(x) = \max\{1-|x|, 0\}$, which attains the minimal Lipschitz constant $1$. The function cannot be represented by a ReLU network with a single hidden layer and finite weights, even in the infinite width limit \cite[Example 5.19]{barron_new}. It can be presented by a neural network with two infinitely wide hidden layers and weight decay $\sim \sqrt d$.

\item The weight decay regularizer of the optimal two-layer ReLU network $f_d^*$ grows like $d$, while its Lipschitz constant grows like $\sqrt d$.  
\end{enumerate}

\item {\bf Highly localized peaks.} The target function can be seen as the prototypical example of learning functions which take values $y_1, \dots, y_N$ at isolated points $x_1, \dots, x_N$ which are separated as `islands' in a `sea' of points $x_{N+1}, \dots, x_M$ with labels $y_{N+1} = \dots = y_M =0$. 

\item {\bf Mollification.} The infinitely wide neural networks constructed in this note can be used to establish approximation rates in function spaces for shallow neural networks by mollification, if the mollification width $\eps$ is optimized to balance the competition between approximation of the target function by the infinitely wide network and approximation of the infinitely wide network by finite neural networks.
\end{enumerate}

To the best of our knowledge, this is the first time that an optimal solution for fitting data by neural networks has been computed in dimension $d>1$. For technical reasons, we focus on the case that $d$ is odd. The optimal radial solution can be written as a finite sum
\[
f_d(x) = \sum_{i=0}^{n+1} \mu_i \avint_{S^{d-1}} \sigma\big(\nu^Tx-b_i\big)\,\d\H^{d-1}, \qquad n = \frac{d-1}2, \quad 0=b_0<\dots<b_{n+1} = 1
\]
for some coefficients $\mu_i\in\R$ satisfying $\sum_{i=0}^n|\mu_i| = \gamma_n \sim 3.7d$. 

The article is organized as follows. In the remainder of the Introduction, we briefly review the context of this work in the literature and the notation we will use throughout the article. In Section \ref{section barron spaces}, we give a brief introduction to the function spaces associated to two-layer ReLU networks with a weight decay regularizer (Barron or Radon BV spaces). Sections \ref{section results} and \ref{section proofs} are dedicated to the statement and proof of our main results respectively. Applications of our results can be found in Section \ref{section applications}. Numerical approximations of the optimal solutions $f_d^*$ can be found in Section \ref{section examples}. We conclude the article with a brief summary and list of open problems in Section \ref{section conclusion}. 

Further numerical experiments can be found in Appendix \ref{appendix plots}. Some proofs from the main part of the article are postponed to Appendix \ref{appendix postponed}, while proofs of results which are known in similar form are postponed to Appendix \ref{appendix known results}. Slight extensions of the main results can be found in Appendix \ref{appendix extensions}.

\subsection{Previous work}

The complexity of a neural network is often measured by the number of its non-zero coefficients (weights) \cite{louizos2017learning,srinivas2017training,gribonval2022approximation} or by a measure of their magnitude. From a practical perspective, both are crucial pieces of information: a neural network with an excessive number of non-zero connections is expensive to store and evaluate, while a network with very large coefficients is likely to depend on subtle cancellations at training data points and unlikely to generalize well to unseen data.

\cite{barron1993universal} realized that a large class of functions can be approximated efficiently by neural networks with a single hidden layer and any sigmoidal activation function while keeping the outer layer coefficients bounded. The function class is defined in terms of a spectral criterion and diverse enough that any linear method of approximation must face the curse of dimensionality in it.

Subsequently, function approximation by ReLU networks with a single hidden layer and bounded coefficients in both layers was studied in \cite{bach2017breaking,E:2018ab,weinan2019lei,barron_new}. Optimal rates of approximation were obtained in \cite{siegel2019approximation,siegel2021optimal}. A spectral criterion for this scenario in terms of the Fourier transform was developed in \cite{klusowski2018approximation}, and a sharp criterion in terms of the Radon transform in \cite{ongie2019function, parhi2021banach}. A detailed study of Fourier-like criteria in this context is given by \cite{barron_boundaries}.

The norm in these function spaces is related to the popular explicit `weight decay' regularizer (the $\ell^2$-norm of the network weights). It retains significance in the context of implicit regularization, as \cite{Chizat:2020aa} showed that infinitely wide two-layer ReLU networks converge to minimum norm/maximum margin classifiers with respect to the weight decay norm, when trained by a gradient flow optimizer for binary classification with logistic loss.

While the structure of the function spaces has been studied and many of their functional analytic properties are understood \cite{barron_new, parhi2021banach, siegel2021optimal,siegel2021characterization}, explicit examples remain rare. Spectral criteria have been used to show that functions in certain smoothness classes can be expressed as infinitely wide two-layer networks with finite weight-decay norm. \cite{weinan2022emergence} construct a maximum margin classifier in a simple one-dimensional scenario. A structure theorem is given in \cite{barron_new} to easily demonstrate that certain functions cannot be expressed this way. Closest to the present work is \cite{hanin2021ridgeless}, where the minimum norm interpolants of a finite one-dimensional data set are studied.

Much of the work on ReLU-activated two-layer networks makes heavy use of the homogeneity of the activation function. Two-layer neural networks with arbitrary activation are studied e.g.\ in \cite{siegel2020approximation, li2020complexity}. Partial (and different) extensions to deeper neural networks can be found e.g.\ in \cite{parhi2022kinds, deep_barron}, while residual neural networks of continuous depth (`neural ODEs') have been studied from this perspective in \cite{weinan2019priori,weinan2019lei,E:2019aa}.

\subsection{Notation}

We denote by $\avint_A$ the average integral over a set $A$ which has finite measure for a measure $\mu$, i.e.\ $\avint_A f(x)\,\d\mu_x = \frac1{\mu(A)} \int_A f(x)\,\d\mu_x$. By $\d\mu_x$ we mean that we integrate with respect to the (signed) measure $\mu$ in the variable $x$. In this article, $\mu$ will always be a measure (often signed), while $\nu$ denotes the exterior normal vector field on a sphere.

The natural $d-1$-dimensional area (Hausdorff) measure is denoted by $\H^{d-1}$. In this article, it will always refer to the (unnormalized) uniform distribution on a $d-1$-dimensional sphere.

The total variation norm of a measure $\mu$ on a measurable space $X$ is defined as $\|\mu\|_{TV} = \mu_+(X) + \mu_-(X)$, where $\mu_+, \mu_-$ is the Hahn decomposition of the signed measure $\mu$.

In the following, $g$ is always going to be a function of one variable and $f$ is going to be a radially symmetric function on $\R^d$. By an abuse of notation, we will also consider $f:[0,\infty) \to \R$ defined by $f(r) = f(r\cdot e_1)$. We denote by
\[
c_d = \frac{|S^{d-2}|}{|S^{d-1}|} = \frac1{\int_{-1}^1 (1-s^2)^\frac{d-3}2\ds}
\]
a quotient related to the area of hyperspheres in dimension $d$ and $d-1$, and by $\gamma_n$ a constant related to the approximability of the function $\sqrt s$ by polynomials of degree at most $n$ in $L^\infty(0,1)$, which also relates to the minimal value of the weight decay regularizer for fitting data as above.

The variables $d$ and $n$ are always related by $n = \frac{d-1}2$, i.e.\ $d = 2n+1$.

\section{Weight decay and Barron spaces}\label{section barron spaces}

In this section, we briefly review the theory of infinitely wide ReLU networks with a single hidden layer. Function spaces for this setting have been studied under the name $\mathcal F_1$ in \cite{bach2017breaking}, Barron space in \cite{review_article,weinan2019lei,E:2019aa,E:2018ab}, Radon-BV in \cite{parhi2022kinds,parhi2021banach} and the convex hull of the ReLU dictionary or the variation space of the ReLU dictionary in \cite{siegel2021characterization, siegel2021sharp}. In this note, we refer to them as Barron spaces in reference to the seminal work of Andrew Barron \cite{barron1993universal}. Some results presented below are extensions of known results to the case where we consider a Barron semi-norm rather than the full Barron norm, corresponding to a weight decay regularizer which does not control the magnitude of the biases.

A neural network with a single hidden layer and $m\in\N$ neurons can be represented as 
\[\showlabel\label{eq finite network}
f_m(x) = \sum_{i=1}^m a_i \,\sigma(w_i^Tx + b_i)\qquad\text{or}\qquad f_m(x) = \frac1m\sum_{i=1}^m a_i \,\sigma(w_i^Tx + b_i)
\]
where $(a_i, w_i, b_i)\in \R\times \R^d\times \R$ are the {\em weights} of the neural network. For networks in which the size of the weights is controlled, this representation can be generalized to
\[\showlabel\label{eq barron representation}
f_\mu(x) = \int_{\R\times \R^d\times \R} a\,\sigma(w^Tx+b)\,\d\mu_{(a,w,b)}\qquad\text{or } f_\pi(x) = \int_{\R\times \R^d\times \R} a\,\sigma(w^Tx+b)\,\d\pi_{(a,w,b)}
\]
where $\mu$ is a measure on $\R^{d+2}$ and $\pi$ is a probability measure on $\R^{d+2}$. More generally, due to the symmetry $a\sigma(w^Tx+b) = \lambda \big((\lambda^{-1} a)\,\sigma(w^Tx+b)\big)$ for $\lambda\neq 0$, $\mu$ can be taken to be a signed measure. Finite networks are contained in the general setting by setting
\[
\mu_m = \sum_{i=1}^m \lambda_i\, \delta_{(\lambda_i^{-1}a_i, w_i, b_i)} \quad\text{and}\quad \pi_m = \frac1m \sum_{i=1}^m \delta_{(a_i, w_i, b_i)} 
\]
respectively, where the parameters $\lambda_i\neq 0$ can be chosen freely for a convenient representation. The integral is guaranteed to converge if the {\em Barron norm} 
\begin{align*}\showlabel\label{eq barron norm}
\|f\|_\B &= \inf_\pi\left\{\int_{\R^{d+2}} |a|\cdot \big\{|w| + |b|\big\}\,\d\pi : f\equiv f_\pi\right\}
	= \inf_\mu\left\{\int_{\R^{d+2}} |a|\cdot \big\{|w| + |b|\big\}\,\d|\mu| : f\equiv f_\mu\right\}.
\end{align*}
is finite, where $|\mu| = \mu^+ + \mu^-$ denotes the total variation measure of the signed measure $\mu= \mu^+-\mu^-$. The infimum must be taken since the representation of a function in this fashion is highly non-unique \cite[Section 2.1]{barron_new}. The two representations of the norm coincide by \cite[Section 2.4]{barron_new}.  

The norm in the parameter variable $w$ is chosen dual to the norm in the data variable $x$ such that the inequality $|w^Tx| \leq |w|\cdot |x|$ holds. In particular, if distances in the data domain are measured in the $\ell^p$-sense for $p\in [1,\infty]$, then distances in the parameter domain are measured in the $\ell^q$-sense for $q = \frac{p}{p-1}$. For compatibility with radial symmetry, we focus on the case $p = \frac{p}{p-1} = 2$ in this note.

We refer to the space $\{f : \|f\|_\B<\infty\}$ as Barron space $\B$, or at times $\B(\R^d)$ to indicate dependence on dimension.

Due to the control over the bias, the Barron norm as defined in \cite{weinan2019lei, barron_new} is not invariant under translations in the data space, i.e.\ the functions $f$ and $f(\cdot +\bar x)$ generally have a different norm for $\bar x\neq 0$. By contrast, the following {\em Barron semi-norm} is translation invariant and has useful properties which suffice in many applications:
\begin{align*}\showlabel\label{eq barron semi-norm}
[f]_\B = \inf_\pi\left\{\frac12\int_{\R^{d+2}}|a|^2 + |w|^2\,\d\pi : f\equiv f_\pi\right\}
	= \inf_\mu\left\{\frac12\int_{\R^{d+2}}|a|^2 + |w|^2\,\d|\mu| : f\equiv f_\mu\right\}.
\end{align*}
We will address the convergence of the integrals in \eqref{eq barron representation} without control over $b$ in Proposition \ref{proposition properties of barron functions}.
This is more in line with the approach in \cite{ongie2019function,parhi2021banach}, where the magnitude of the bias is also not controlled. We opt for controlling $|a|^2 + |w|^2$ rather than $|a|\cdot |w|$ for convenience, but note that the classical Barron norm could be defined in this fashion, too. The key observation is that the ReLU activation function $\sigma(z) = \max\{z,0\}$ is positively one-homogeneous, i.e.\ $\sigma(\lambda z) = \lambda \sigma(z)$ for all $\lambda >0$. In particular
\[
a \sigma(w^Tx+b) = a \sqrt{\frac{|w|}{|a|}}\,\sigma\left(  \sqrt{\frac{|a|}{|w|}} w^Tx+\sqrt{\frac{|a|}{|w|}}b\right),
\]
i.e.\ we may normalize neurons $(a_i,w_i,b_i)$ to
\[
a_i' = a_i \sqrt{\frac{|w_i|}{|a_i|}}, \quad w_i' = \sqrt{\frac{|a_i|}{|w_i|}} \,w_i\qquad\text{s.t. } |a_i'|^2 = |w_i'|^2 = |a_i|\,|w_i|
\]
without changing the output of the neural network. In particular $|a|^2 + |w|^2 = 2|a|\,|w|$, indicating that we could define the Barron norm in the analogous fashion by squares. Indeed, in the infinite limit it is even possible to assume that $\pi$ is supported on the set $|a| = |w| = \sqrt{[f]_\B}$. For a more technically rigorous discussion, see e.g.\ \cite{barron_new}.

By a slight abuse of terminology, we will also refer to $\B_0$ as Barron space from now on. We briefly note the following properties, which relate the Barron semi-norm and more well-established quantities.

\begin{proposition}\label{proposition properties of barron functions}
\begin{enumerate}
%\item For any $f\in \B_0$, there exists $\pi$ such that $f= f_\pi$ and $\pi$ realizes the infimum in \eqref{eq barron semi-norm}.
\item If the integral in \eqref{eq barron semi-norm} is finite for $\pi$, then the integral defining $f_\pi$ in \eqref{eq barron representation} exists for all $x\in \R^d$ if and only if it exists for $x=0$. It may then be re-cast as
\[
f_\pi(x) = f_\pi(0) + \int_{\R^{d+2}} a\,\big[\sigma(w^Tx+b) - \sigma(b)\big] \,\d\pi.
\]
This expression always converges if the integral in \eqref{eq barron semi-norm} is finite. The integral exists as a Bochner integral with values in $C^0(K)$ for compact $K\subseteq \R^d$ or $L^p(\P)$ for a probability distribution $\P$ on $\R^{d+2}$ with finite $p$-th moments.
\item $[f]_\B$ is a norm on the modified Barron space $V_0 = \{f\in C^0(\R^d): f(0) = 0, [f]_\B<\infty\}$, which makes $V_0$ a Banach space. Compared to classical Barron spaces, $V_0\not\subseteq \B(\R^d)$.
\item $[f]_\B\leq \|f\|_\B$.
\item If $f\in\B$, then $f$ is Lipschitz-continuous and the Lipschitz-constant of $f$ satisfies $[f]_{Lip}\leq [f]_\B$.
\end{enumerate}
\end{proposition}

All statements could be given in terms of a general signed measure $\mu$ instead of $\pi$. The proof, along with other proofs from this section, can be found in Appendix \ref{appendix known results}.

Functions in Barron spaces are defined by means of an explicit representation formula. Paradoxically, this explicit characterization often makes it difficult to verify whether a given function is in Barron space. A more abstract framework was created in \cite{ongie2019function} by the means of the Radon transform, based on the observation that 
\begin{align*}
\Delta \left(\sum_{i=1}^m a_i\sigma(w_i^T\cdot +b_i)\right) &= \sum_{i=1}^m a_i |w_i|\cdot \H^{d-1}|_{W_i}\\ 
D^2 \left(\sum_{i=1}^m a_i\sigma(w_i^T\cdot +b_i)\right) &= \sum_{i=1}^m a_i |w_i|\cdot \frac{w_i}{|w_i|} \otimes \frac{w_i}{|w_i|}\cdot \H^{d-1}|_{W_i},
\end{align*}
i.e.\ the second spatial derivatives of a ReLU network with one hidden layer are superpositions of measures concentrated on the hyperplanes $W_i = \{x : w_i^Tx + b_i = 0\}$. This allows for a characterization of Barron spaces in terms of second derivatives. The Radon transform is used as a technical tool in order to dualize from hyperplanes to points. This convenient characterization allows the construction of some examples of functions in Barron space.

\begin{example}
\begin{enumerate}
\item Assume that $f$ is a Lipschitz-continuous function and that the (possibly non-integer) power $(-\Delta)^{(d+1)/2}f$ of the Laplacian in the distributional sense exists as a measure. Then
\[
[f]_{\B(\R^d)} \leq \frac1{2^{d-1} \pi^{d/2-1} \Gamma(d/2) } \|(-\Delta)^{(d+1)/2} f\|_{TV},
\]
where $\|\cdot\|_{TV}$ denotes the total variation norm of $\Delta f$ \cite[Proposition 3]{ongie2019function}.

\item If $d$ is odd, the power of the Laplacian is integer. In particular, if $f$ belongs to the Sobolev space $W^{d+1,1}(\R^d) \subseteq C^{d+1}(\R^d)$ of functions whose first $d+1$ (weak) partial derivatives are $L^1$-integrable, then $f\in \B(\R^d)$ and 
\[
[f]_\B \leq c_d \|f\|_{W^{d+1,1}}.
\]
for some constant $c_d>0$, which depends on the exact choice of the norm on $W^{d+1,1}$. In particular $C_c^\infty(\R^d)\subseteq \B(\R^d)$ \cite[Corollary 1]{ongie2019function}.

\item If $d\geq 3$ is an odd integer and $f_{d,k}:\R^d\to \R$ is the radial bump function given by 
\[
f_{d,k}(x) = \begin{cases} \big(1- |x|^2\big)^k &|x|\leq 1\\ 0 &\text{else}\end{cases},
\]
then $f_{d,k} \in\B_0(\R^d)$ if $k\geq \frac{d+1}2$. For $k_d = \frac{d+1}2 + 2$, the norm bound $[f_{d,k_d}]_{\B(\R^d)}\leq 2d(d+5)$ holds according to \cite[Example 3]{ongie2019function}. 

In \cite{ongie2019function}, also a stronger version of the statement is claimed, including an if and only if condition for $k$ and a comparable lower bound for $[f_{d,k_d}]_{\B(\R^d)}$. Those claims are based on an error in the proof of \cite[Proposition 15]{ongie2019function}, where the erroneous claim is made that if $\int_{\R^d}|\phi|\dx =1$, then the integral of $\phi$ over any hyperplane $\int_H|\phi|\,\d\H^{d-1}$ is bounded from above by $1$. 
\end{enumerate}
\end{example}

Based on the same intuition, we point out two observations. The first demonstrates that the singular set $\Sigma$ of a Barron function (i.e.\ the set where the functions is not differentiable) is `straight' and lower dimensional. This is a stronger version of Rademacher's theorem, which states that the singular set of a Lipschitz function is Lebesgue null, in the context of Barron spaces. The following statement has the stronger implication that $\Sigma$ is contained in a countable union of affine subspaces of $\R^d$ and therefore has Hausdorff dimension $\leq d-1$.

\begin{proposition}\cite{barron_new}\label{proposition structure theorem}
Any function $f\in\B(\R^d)$ can be written as a countable sum $f = \sum_{i=0}^\infty f_i$ where 
\begin{enumerate}
\item $f_0\in\B(\R^d)$ is $C^1$-smooth,
\item $f_i(x) = g_i(P_ix+b_i)$ where
\begin{itemize}
\item $P_i:\R^d\to \R^{k_i}$ is an orthogonal projection for $1\leq k_i\leq d$ (i.e.\ $P_iP_i^T = I_{k\times k})$
\item $g_i\in \B(\R^{k_i})$ is $C^1$-smooth except at $0\in \R^{k_i}$.
\end{itemize}
\end{enumerate}
\end{proposition}

The fact that the singular set is straight has two immediate implications.

\begin{corollary}\label{corollary structure barron}
\begin{enumerate}
\item If $f\in \B(\R^d)$ is radially symmetric, then $f\in C^1(\R^d\setminus \{0\})$. 
\item If $\phi:\R^d\to\R^d$ is a diffeomorphism such that $f\in \B(\R^d) \Ra f\circ \phi\in \B(\R^d)$, then $\phi$ is an affine linear map \cite[Theorem 5.18]{barron_new}.
\end{enumerate}
\end{corollary}

A brief inspection of the proof of \cite[Theorem 5.18]{barron_new} reveals that Proposition \ref{proposition structure theorem} and \ref{corollary structure barron} reveals that both statements remain valid for $\B_0(\R^d)$. A stronger result on radial Barron functions is proved below in Lemma \ref{lemma radial barron functions}. Secondly, we recall a characterization of one-dimensional Barron spaces, which is essentially the simpler one-dimensional case of the Radon transform construction. A similar statement can also be found e.g.\ in \cite[Example 4.1]{barron_new} and \cite{li2020complexity}.

\begin{proposition}\label{proposition one-dimensional}
$\phi \in\B_0(\R)$ if and only if there exists a finite signed measure $\mu$ such that $\phi'' = \mu$, i.e.\ $\phi'(s) = \mu\big((-\infty, s]\big)$ for all $s\in\R$ such that $\mu(\{s\}) = 0$ (in particular, all but countably many). For all such $\phi$ and any $a\in\R$, we can write
\[
\phi(z) = \phi(a) + \phi'(a) \big[\sigma(x-a) - \sigma(a-x)\big] + \int_a^\infty \phi''(s) \,\sigma(z-s)\ds + \int_{-\infty}^a \phi''(s)\,\sigma(s-z)\ds.
\]
Furthermore 
\[
[\phi]_\B \leq \|\phi''\|_{TV} + 2 \inf_{a\in\R} \inf_{v\in \partial \phi(a)} |v|
\]
where 
\[
\partial_af = \mathrm{conv}\left(\left\{v\in \R : \exists\ x_n\to a \text{ s.t. } \frac{f(x_n) - f(a)}{x_n-a}\to v\right\}\right)
\]
is the convex hull of the set of approximate derivatives. Conversely
\[
\max\left\{ \|\phi''\|_{TV},  \sup_{a\in\R} \inf_{v\in \partial \phi(a)} |v|\right\} \leq [\phi]_\B.
\]
\end{proposition}

We believe that the upper bound is, in fact an identity. We now recall a property of Barron spaces $\B_0$.

\begin{proposition}[Direct approximation theorem]\label{proposition direct approximation}
For every $f\in \B_0$ and every probability measure $\P$ on $\R^d$ there exists $f_m$ as in \eqref{eq finite network} and $c>0$ such that
\[
\|f- f_m -c\|_{L^2(\P)}^2 \leq \frac{[f]_\B^2}m \: \max_{|\nu|\leq 1}\int_{\R^d}|\nu^Tx|^2\,\d\P
\]
and
\[
 |a_i| = |w_i| = \sqrt{\frac{\|f\|_\B}m}\qquad\text{or } |a_i| = |w_i| = \sqrt{{\|f\|_\B}},
\]
depending on the normalization in \eqref{eq finite network}.
\end{proposition}

For the sake of completeness, we sketch a probablistic proof in Appendix \ref{appendix known results}. This formulation of the direct approximation theorem improves on known results in two major ways:

\begin{enumerate}
\item The dependence on the data distribution $\P$ is only through the `projected second moments' $M_{2,proj}(\P):= \max_{|w|\leq 1}\int_{\R^d}|w^Tx|^2\,\d\P$ rather than the full second moments $M_2(\P):= \int_{\R^d}|x|^2\,\d\P$. It is easy to see that
\[
M_{2,proj}(\P)\leq M_2(\P) = \sum_{i=1}^d\int_{\R^d}|e_i^Tx|^2\,\d\P \leq d\cdot M_{2,proj}(\P)
\]
for any probability measure $\P$ on $\R^d$, and that equality is attained for any measure $\P$ which is the product of $d$ one-dimensional probability measures, e.g.\ a standard normal distribution. The constant in the bound may therefore be significantly smaller in high dimension.

\item The bound depends on the Barron semi-norm, but not the full Barron norm.
 \end{enumerate}

While the constants are improved in this formulation compared to e.g.\ \cite{E:2018ab,weinan2019lei,barron_new}, the result is not expected to be sharp in terms of the rate which is achieved. An improvement from $m^{-1/2}$ to $m^{-1/2 - 3/2d}$ in the classical setting can be found in \cite{siegel2021optimal} at the cost of a more involved proof.

Many of the results above are somewhat specific to ReLU activation as the proofs either use positive homogeneity or the property that $\sigma'' = \delta$. Both are shared by leaky ReLU activation. 

\begin{remark}
Consider the leaky ReLU activation function $\sigma_\eps(z) = \max\{\eps z, z\}$ for $\eps\in(0,1)$ in addition to the classical ReLU activation $\sigma = \sigma_0$. Since
\[\showlabel\label{eq relu and leaky relu}
\sigma_\eps(z) = \sigma(z) - \eps\,\sigma(-z)\qquad \text{and}\qquad \sigma(z) = \frac{1}{1-\eps^2} \sigma_\eps(z) + \frac\eps {1-\eps^2}\sigma_\eps(-z),
\]
any function which can be represented as a superposition of ReLUs can be represented as a superposition of leaky ReLUs and vice versa. The entire construction of Barron space goes through as above, leading to two semi-norms $[\cdot]_\B$ and $[\cdot]_\eps$ on the same function class such that $[f]_\eps \leq (1+\eps ) [f]_\B$ and $[f]_\B \leq \frac{1+\eps}{1-\eps^2} [f]_\eps = \frac{1}{1-\eps}\,[f]_\eps$ by the explicit representation \eqref{eq relu and leaky relu}. More compactly, we write this as
\[\showlabel\label{eq norm inequality relu leaky relu}
(1-\eps)\,[f]_\B \leq [f]_\eps \leq (1+\eps)\,[f]_\B \qquad \forall\ f \in \B_0.
\]
Using positive one-homogeneity, it can be seen that the coefficients in the representations \eqref{eq relu and leaky relu} are in fact optimal and thus that \eqref{eq norm inequality relu leaky relu} is sharp. The norms induced on Barron space by ReLU and leaky ReLU activation are therefore equivalent, and all properties mentioned above survive if $\sigma$ is replaced by $\sigma_\eps$.

The more subtle statements which we prove below do not survive passing to an equivalent norm. When minimizing $[f]_\eps$ under the constraints $f(x_i) = y_i$, the set of solutions $\M_\eps \subseteq \B_0$ will generally depend on $\eps\in[0,1)$. For example, consider the one-dimensional data set with two points $(x_0, y_0)= (0,0)$ and $(x_1,y_1) = (1,1)$, which is fit exactly by $\sigma_\eps$ for any $\eps$. The solution $\sigma_\eps$ is norm-minimizing for $[\cdot]_\eps$, but not for $[\cdot]_\B$, where the norm-inequality is sharp (and vice versa). 

The equivalence of norms estimate degenerates at $\eps = 1$, where the activation would become linear. If $\eps<0$, a similar construction holds unless $\eps=-1$, where the any $\sigma_\eps$-Barron function $f$ must satisfy $\lim_{t\to\infty} f(tx) =  \lim_{t\to-\infty} f(tx)$.
\end{remark}

We are finally ready to state (and prove) the main results of this article rigorously.

\section{Statements of Main Results}\label{section results}

\begin{theorem}\label{theorem main 1}
For every odd $d\in \N$, there exists a unique radial function $f_d^*\in \B(\R^d)$ such that 
\[
f_d^*\in \argmin_{f\in\F} [f]_\B, \qquad \mathcal F:= \left\{f\in C(\R^d) : f(0) = 1 \text{ and } f\equiv 0\text{ on }\R^d\setminus B_1(0)\right\}.
\]
Furthermore
\begin{enumerate}
\item $f_d^*\in C^\frac{d-1}2(\R^d\setminus\{0\})$.
\item The radial profile $\hat f_d^*:[0,\infty)\to \R$, $\hat f_d^*(r) = f_d^*(r\cdot e_1)$ is strictly monotone decreasing in $r$ in $(0,1)$. In particular, $0\leq f_d^*\leq 1$.
\item There exists $r_d>0$ such that $\hat f_d^*$ is a {\em linear}, strictly monotone decreasing function of $r$ on $[0,r_d]$.
\end{enumerate}
As $d\to \infty$, the norm of $f_d^*$ increases  linearly as
\[
\lim_{d\to\infty, \:d\text{ odd}} \frac{[f_d^*]_{\B(\R^d)}} d = \gamma \approx 3.6,
\]
where $\gamma$ is the inverse of the Bernstein constant.
\end{theorem}

The Bernstein constant is a quantity in classical numerical analysis and approximation theory arising when approximating the function $h(x) = |x|$ by polynomials in $L^\infty(-1,1)$, see e.g.\ \cite{trefethen2019approximation}.
From the proof of Theorem \ref{theorem main 1}, we obtain an algorithm to compute $f_d^*$ to arbitrary precision, which is implemented in Section \ref{section examples}.

The functions $f_d^*$ are radially symmetric, compactly supported and non-negative. In particular, they can serve as mollifiers to easily prove quantitative approximation results for two-layer ReLU networks in general function classes. In a companion article \cite{wojtowytsch_radial_empirical}, we prove that they are achieved as (radial averages of) empirical risk minimizers with a weight decay regularizer. 

Note that we do not exclude the possibility that other minimizers exist which are not radially symmetric. From direct arguments, we can only conclude that the set of minimizers is convex and invariant under coordinate rotations. The existence of at least one radially symmetric minimizer follows relatively easily, while its uniqueness is established below by construction. For {\em any} minimizer $\tilde f_d\in \argmin_{f\in \F}[f]_\B$, which may not not be radially symmetric, the radial average
\[
\tilde f_{d,av}(x) = \avint_{SO(d)} \tilde f_d(Ox)\d H_O
\]
is a radially symmetric minimizer, i.e.\ $\tilde f_{d,av}\equiv f_d^*$. Knowledge of the unique minimizer after radial averaging allows us to study optimization algorithms for implicit bias and finding global optima. This line of inquiry is pursued in upcoming work \cite{ian_juyoung}.

We find it easier to deal with odd dimensions, as the function $(1-s^2)^{\frac{d-1}2}$ is a polynomial in this case. This is analogous to the observations of \cite{ongie2019function}. We remark that, if $f:\R^D\to\R$ is a Barron function and $d\leq D$, then 
\[
\widetilde f:\R^d\to\R, \qquad \widetilde f(x) = f(x_1,\dots, x_d, 0,\dots, 0)
\]
is also a Barron function and $[\widetilde f]_{\B(\R^d)} \leq [f]_{\B(\R^D)}$, so the limit
\[
\lim_{d\to \infty} \inf\left\{[f]_{\B(\R^d)} : f(0)=1\text{ and } f\equiv 0\text{ on }\R^d\setminus\overline{B_1(0)}\right\} \approx 3.6
\]
remains valid if even dimensions are considered, as can be seen when sandwiching an even integer $d$ between $d-1$ and $d+1$.

Using a reflection argument, any radially symmetric Barron function $f:\R^d\to\R$ can be written as
\[\showlabel\label{eq generic radial barron function}
f(x) = f(0) + \int_{[0,\infty)} \avint_{S^{d-1}} \sigma (\nu^Tx-b)\d\H^{d-1}\,\d\mu_b
\]
for some measure $\mu$ on the space of biases. In this context, Theorem \ref{theorem main 1} can be understood as a finite representer theorem, since the proof shows precisely that there exist $n+2 = \frac{d+3}2\in\N$ weights $\mu_0, \dots, \mu_{n+1}$ and biases $0 = b_0 < \dots < b_{n+1} = 1$ such that
\[
f_d^*(x) = 1+ \sum_{i=0}^{n+1} \mu_i \avint_{S^{d-1}} \sigma (\nu^Tx-b_i)\d\H^{d-1}.
\]
Finally, we note that the methods in the proof of Theorem \ref{theorem main 1} can also be used to show the following extension.

\begin{theorem}\label{theorem main 3}
For every $\eps\in(0,1)$ and every odd $d\in \N$, there exists a unique radial function $f_{d,\eps}^*\in \B(\R^d)$ which minimizes the Barron semi-norm in the class
\[
f_{d,\eps}^*\in \argmin_{f\in\F_\eps} [f]_\B, \qquad 
\mathcal F_\eps:= \left\{f\in C(\R^d) :f\equiv 1 \text{ on }\overline{B_\eps(0)} \text{ and } f\equiv 0\text{ on }\R^d\setminus B_1(0)\right\}.
\]
Furthermore
\begin{enumerate}
%\item $f_{d,\eps}^*$ is radially symmetric.
\item $f_{d,\eps}^*\in C^\frac{d-1}2(\R^d)$.
\item The radial profile $\hat f_{d,\eps}^*:[0,\infty)\to \R$, $\hat f_{d,\eps}^*(r) = f_{d,\eps}^*(r\cdot e_1)$ is strictly monotone decreasing on $[\eps,1]$. In particular, $0\leq f_{d,\eps}^*\leq 1$.
\end{enumerate}
In this case, the Barron norm grows exponentially in the dimension $d$. More precisely, there exists $D\in \N$ independent of $\eps>0$ such that
\[
\|f_{d,\eps}^*\|_{\B(\R^d)} \geq \,\frac{\eps^2\,\sqrt d} {(1-\eps^2)^\frac{d+1}2}
\]
if $d\geq D$.
\end{theorem}

We thus observe that the problem of approximating compactly supported bump functions which are constant in a neighbourhood of the origin by shallow neural networks suffers from the curse of dimensionality. We will argue below that this is not the case for ReLU networks with at least two hidden layers.

\section{Proofs of the Main Results\label{section proofs}}

We begin by stating three lemmas in this section, which are used to prove the main theorems. The proofs are given in Appendix \ref{appendix postponed}.
By a slight abuse of notation, we denote $f(r) = f(re_1)$ for a radially symmetric function $f:\R^d\to\R$ and by $f'$ the radial derivative of $f$. We first note a general result on radially symmetric Barron functions.

\begin{lemma}\label{lemma radial barron functions}
Let $f:\R^d\to\R$ be a radially symmetric Barron function and $d$ odd. Then
\begin{enumerate}
\item as a function of $r$, $f$ is $n:= \frac{d-1}2$ times continuously differentiable in $\R^d\setminus\{0\}$. The $n+1$-th radial derivative is bounded and measurable, and the $n+2$-th radial derivative in the distributional sense is a bounded (Radon) measure.
\item for every $\eps>0$, there exists $D\in \N$ such that the Lipschitz bound
\[
 [f]_{Lip}\leq \frac{1+\eps}{\sqrt{2\pi d}} [f]_{\B_0}
 \]
 holds for every $d\geq D$.
\end{enumerate}
\end{lemma}

The following Lemma allows us to express radial symmetry and compact support for Barron functions in odd dimensions in a one-dimensional fashion. It is based on an exchange in the order of integration in \eqref{eq generic radial barron function}.

\begin{lemma}\label{lemma 1d reduction}
Assume that $g\in \B_0 (\R)$ is a one-dimensional Barron function such that $g(0) = 1$, $g\equiv 0$ outside of $(-1,1)$ and
\[\showlabel\label{eq orthogonality conditions}
\int_{-1}^1g(s) \,s^{2k}\ds =0\qquad \text{for } k = 1,\dots, \frac{d-3}2.
\]
Then the function $f:\R^d\to\R$ given by
\[\showlabel\label{eq f from g}
f(x) = \frac{\int_{-1}^1(1-s^2)^\frac{d-3}2\,g(|x|s)\,\ds}{\int_{-1}^1(1-s^2)^\frac{d-3}2\,\ds}
\]
satisfies the following properties:
\begin{enumerate}
\item $f(0) = 1$,
\item $f(x) = 0$ if $|x|\geq 1$,
\item $f$ is radially symmetric, and
\item $[f]_{\B(\R^d)} \leq [g]_{\B(\R)}$.
\end{enumerate}
Conversely, if $f:\R^d\to\R$ is a radially symmetric Barron function which satisfies $f(0) = 1$ and $f\equiv 0$ on $\R^d\setminus B_1(0)$, then there exists $g$ as above such that \eqref{eq f from g} holds, which is additionally an even function and satisfies $[f]_{\B(\R^d)} = [g]_{\B(\R)}$.

Furthermore $f\equiv 1$ in $B_\eps(0)$ if and only if $g\equiv 1$ in $(-\eps,\eps)$ for the {\em even} representative of the function class $g$. 
\end{lemma}

We will show that such a Barron function $g$ indeed exists for every odd $d\geq 3$ and compute the precise asymptotic growth of $[g]_{\B}$ as $d\to\infty$. 
The following Lemma is the main technical tool in our proof.

\begin{lemma}\label{lemma auxiliary bernstein}
For $n\in \N$, set 
\[
\gamma_n:= \min \left\{ \|\mu\|_{TV} :  \int_{0}^1 s\,\d\mu_s =1, \quad \int_{0}^1 s^{2k}\,\d\mu_s=0\text{ for }0\leq k\leq n\right\}.
\]
Then $\lim_{n\to \infty} \frac{\gamma_n}n= \gamma \approx 3.57$ is the inverse of the Bernstein constant. 
The minimum is attained by a unique measure $\mu = \sum_{i=0}^{n+1} \mu_i \delta_{s_i}$ where
\begin{enumerate}
\item $0 = s_0 < s_1< \dots < s_{n+1} = 1$ are the $n+2$ distinct points in $[0,1]$ at which $P(s) -s$ is extremal in $[0,1]$, where $P$ is the optimal even polynomial approximator of degree $\leq 2n$ for $g(s) =s$ in $L^\infty(0,1)$.
\item $\mu_0,\dots, \mu_{n+1}\in\R$ are parameters satisfying the alternation criterion $\mu_{i+1}\mu_i<0$ for $i=0, \dots, n$ and the generalized Vandermonde system
\[
\begin{pmatrix} s_0 & s_1& \dots &s_{n+1}\\1&1& \dots &1\\  s_0^2 & s_1^2 & \dots &s_{n+1}^2\\ s_0^4 & s_1^4 & \dots &s_{n+1}^4\\ \vdots &\vdots &\ddots &\vdots\\ s_0^{2n} & s_1^{2n} & \dots &s_{n+1}^{2n}\end{pmatrix} 
\begin{pmatrix} \mu_0\\ \vdots\\ \mu_{n+1}\end{pmatrix} = \begin{pmatrix} 1\\ 0 \\ \vdots\\ 0 \end{pmatrix}.
\]
\end{enumerate}
\end{lemma}

We are finally prepared to prove our main results.

\begin{proof}[Proof of Theorem \ref{theorem main 1}]
{\bf Step 1.} In this step, we construct $g$ and $f$ using Lemmas \ref{lemma 1d reduction} and \ref{lemma auxiliary bernstein}. 
Let $d\geq 3$ be an odd integer and $n= \frac{d-1}2$. Let $g:[0,\infty)\to\R$ be the unique function such that 
\[
g(0) = 0, \qquad \lim_{z\nearrow -1}g'(z) = 0, \qquad g'' = \mu,
\]
in the distributional sense, where $\mu$ is the even reflection of the measure $\mu_{n+1}$ described in Lemma \ref{lemma auxiliary bernstein}, i.e.\ $\mu(U) = \mu_{n+1}(U\cap [0,\infty)) + \mu_{n+1}(-U\cap [0,\infty))$. Note that the origin is counted twice. By construction, $g$ is piecewise linear, $g\equiv 0$ on $(-1,1)$ and
\[
g(s) = \int_{-1}^s (z-s)\,\d\mu_z \qquad \forall\ s\geq -1.
\]
In particular, $g(0) = 1$ and $g(s) = 0$ for all $s\geq 1$ due to the moment conditions
\[
\int_0^1 1\,\d\mu_s = 0, \qquad \int_0^1 s\,\d\mu_s = 1.
\]
Since $g''=\mu$ is even and $g(-1) = g(1)$, we find that $g$ is even. Due to Proposition \ref{proposition one-dimensional}, we observe that $[g]_\B = \|\mu\|_{TV} = 2\gamma_{n+1}$.
Integrating by parts twice, we realize that
\[\showlabel\label{eq by parts for mu}
\int_{-1}^1 g(s)\,s^{2k} \ds = \int_{-r}^r g(s)\,s^{2k}\ds = \frac1{(2k+2)(2k+1)} \int_{-r}^r g''(s)\,s^{2k+2}\ds =0
\]
for $r>1$ and $k= 0,\dots, n-1= \frac{d-3}2$, since $g\equiv g'\equiv 0$ in a neighbourhood of $r$, so the boundary terms vanish. The integration by parts is well-established for smooth functions and can be justified in the piecewise case by mollification. 

In particular, $g$ satisfies the conditions of Lemma \ref{lemma 1d reduction} and induces an admissible radially symmetric function $f\in \B_0(\R^d)$.

{\bf Step 2.} Assume for now that there exists a minimizer of the Barron semi-norm in $\F$. Since the Barron semi-norm is a a convex function on the convex function class $\F$, and since furthermore both $\F$ and $[\cdot]_\B$ are invariant under coordinate rotations, we note that the set of minimal semi-norm elements in $\F$ is both convex and rotation invariant. In particular
\[
\hat f \in \argmin_{f\in \F} [f]_\B \qquad \Ra \quad \hat f_O \in \argmin_{f\in \F} [f]_\B \quad\text{where } \hat f_O(x) = \int_{SO(d)} \hat f(Ox)\,\d H_O 
\]
is the average of $\hat f$ with respect to all rotations. The measure $H$ is the Haar measure on the group $SO(d)$, i.e.\ the $d(d-1)/2$-dimensional Hausdorff measure induced by the Frobenius norm on the space of $d\times d$-matrices.

Thus, if a minimizer exists, then there is also a radially symmetric minimizer. In this step, we illustrate that the function $f = f_d^*$ associated to $g$ as in Step 1 is in fact optimal. In particular, we can conclude from the proof below that a minimizer does exist..

It is easy to see that \eqref{eq by parts for mu} is both necessary and sufficient to imply the moment conditions for $g$ in Lemma \ref{lemma 1d reduction}. In particular
\[
\inf_{f\in \F}[f]_\B = \inf\left\{[g]_\B : g(0) = 1, \:g \equiv 0\text{ on }[1,\infty), \:g\text{ even and }\int_0^1 g''(s)\,s^{2k}\ds = 0\text{ for }0\leq k \leq \frac{d-1}2\right\},
\]
with corresponding minimizers. As $g''=\mu$ is the unique solution to the minimization problem on the right, $f$ is the unique radial minimizer on the left.

In particular
\[
\lim_{d\to\infty, d\text{ odd}} \frac{[f_d^*]_\B}d = \lim_{d\to\infty, d\text{ odd}} \frac{2 \gamma_{(d-1)/2}}d = \gamma \leq 3.6.
\]
{\bf Step 3.} We note that $g$ is linear on the interval $[0,s_1]$, where $s_1$ is as in Lemma \ref{lemma auxiliary bernstein}. Thus 
\[
f(x) = c_d \int_{-1}^1 g(|x|s)\big(1-s^2\big)^\frac{d-3}2\ds =  c_d \int_{-1}^1 \big(1-\mu_0|x|\,|s|\big)\big(1-s^2\big)^\frac{d-3}2\ds = 1 - |x| \frac{\mu_0\int_{0}^1s\big(1-s^2\big)^\frac{d-3}2\ds}{\int_0^1\big(1-s^2\big)^\frac{d-3}2\ds}
\]
is linear by the origin. 

{\bf Step 4.} In this step, we show that $f$ is strictly decreasing in radial direction inside the unit ball. As noted in Corollary \ref{corollary structure barron}, the function $f$ is $C^1$-smooth away at the origin. Since $f(0) = 1$ and $f(e_1)= 0$, it suffices to show that $\partial_rf (re_1)\neq 0$ for $r\in (0,1)$. We compute
\[
f(re_1) = c_d\int_{-1}^1 g(rs) \,\big(1-s^2\big)^\frac{d-3}2\ds, \qquad \partial_rf (re_1) =c_d \int_{-1}^1 g'(rs) \,s\big(1-s^2\big)^\frac{d-3}2\ds.
\]
We make the following claim: {\em If $g$ is an even piecewise linear function on $[-1,1]$ with at most $n+1$ segments in $[0,1]$ and $k\leq n-1$, then the function
\[
r\mapsto \int_{-1}^1 g'(rs) \,s\big(1-s^2\big)^k\ds
\]
has at most $n-2-k$ zeros in $(0,1)$.}

To prove the claim, start with $k=0$. Then 
\[
\int_{-1}^1 g'(rs)\,s \ds = \frac1{r^2} \int_{-1}^1g'(rs)\,rs \,r\ds = \frac{2}{r^2}\int_0^r g'(z)\,z\dz.
\]
As $g'$ is constant in the interval $[0,s_1]$ by the origin, $\partial_rf$ is constant (and non-zero) in $[0,s_1]$, meaning that $\partial_rf$ cannot have a zero in $[0,s_1]$.
In any interval $(s_i, s_{i+1})$ where $g'$ is constant, the function
\[
r\mapsto \int_0^r g'(z)\,z\dz = \int_0^{s_i} g'(z)\,z\dz + \int_{s_i}^r g'(z)\,z\dz
\]
is monotone, since $g'(z)\cdot z$ does not change sign. In particular:
\begin{enumerate}
\item There is no zero in the first interval $[s_0, s_1] = [0, s_1]$.
\item There is at most one zero in $[s_i, s_{i+1}]$. 
\item The zero in the final interval $[s_n, s_{n+1}] = [s_n, 1]$ is attained at $s=1$.
\end{enumerate} 
Thus there are at most $n-2$ zeros in $(0,1)$, which proves the claim for $k=0$. Now consider $k\geq 1$. Note that
\[
\int_0^1 g'(rs)\,s(1-s^2)^k\ds = r^{-(2+k)} \int_0^1 g'(rs)\,rs\big(r^2-(rs)^2\big)^k\,r\ds = r^{-(2+k)} \int_0^r g'(z)\big(r^2-z^2\big)^k\dz 
\]
In particular, the term on the left is zero if and only if the integral on the right is zero. If there are two points $r_1, r_2$ on the right where the integral vanishes (and $k\geq 1$), then by Rolle's theorem in between there exists a point $r\in (r_1,r_2)$ at which
\[
0 = 2r\int_0^r g'(z)\,z \big(r^2 - z^2\big)^{k-1}\dz = r^{k-2} \int_{-1}^1 g'(rs)\,s(1-s^2)^{k-1}\ds
\]
The integral also vanishes at zero, where we are integrating over the empty set. We note that for any $k\geq 0$ the integral $\int_{-1}^1 g'(rs)\,s (1-s^2)\ds$ vanishes at $r=0$ and $r=1$. If, for $k\geq 1$ it vanishes at $N$ interior points, then for $k-1$ it must vanish at $N+1$ points: 0, 1, and at least once in each interval. In particular, for $k\geq 1$, there are at most $n-2-k$ interior vanishing points.

{\bf Step 5.} We finally note that the Lipschitz bound follows directly from the Barron norm bound on $f_d^*$ and Lemma \ref{lemma radial barron functions}.
\end{proof}

\begin{example}\label{example d=3}
Let us consider the case $d=3$, i.e.\ $n= \frac{d-1}2= 1$. The $n+2 =3$ points $s_0, s_1, s_2$ are given by the equi-oscillating points of the best approximation of the function $f(s) = s$ on $[0,1]$ by elements of the space spanned by $\{s^0, s^2, \dots, s^{2n}\}= \{1,s^2\}$.

The best approximation of $s$ by even quadratic polynomials in $L^\infty(0,1)$ is $P(s) = s^2 + \frac18$, which attains maximal distance at $s= 0, 1/2, 1$. This can easily be verified as $P(s) - s$ is a polynomial of degree $2$ inside $(0,1)$, so if $P(0) = P(1)$, then $P(s) = \alpha + \beta(s-1/2)^2$, so the most distant points are in $\{0,1/2,1\}$. By Kolmogorov's equi-oscillation theorem, all three are points of largest error. It is now easy to solve for the coefficients of $P$.

%It therefore suffices to observe that
%\[
%P(0) - 0 = \frac 18, \qquad P(1/2)-1/2 = \frac14 + \frac18-\frac12 = -\frac18
%\]
%and
%\[
%P(1-s) - (1-s) = (1-s)^2 + \frac18 - (1-s) = 1 -2s + s^2 + \frac18 -1 +s = s^2+ \frac18 -s = P(s)-s.
%\]
We can find the measure $\mu = \mu_0\delta_0 + \mu_1 \delta_{1/2} + \mu_2\delta_1$ for the second derivative $g'' = \mu$ by solving the linear system of moment conditions
\[\showlabel \label{eq modified vandermonde 3d}
\begin{pmatrix} 1&1&1\\ 0&1/2&1\\ 0&1/4 &1\end{pmatrix} \begin{pmatrix} \mu_0\\ \mu_1 \\ \mu_2\end{pmatrix} = \begin{pmatrix} 0\\ 1\\ 0\end{pmatrix} \quad\LRa \quad \begin{pmatrix} \mu_0\\ \mu_1\\\mu_2 \end{pmatrix} = \begin{pmatrix} -3\\ 4\\ -1\end{pmatrix}.
\]
So $g$ is the even continuous piecewise linear function satisfying $g(0) =1$ and
\[
g'(s) = \begin{cases} \mu_0 &s\in(0,1/2)\\ \mu_0+ \mu_1 & s\in(1/2,1)\\ \mu_0 + \mu_1+\mu_2 &s>1\end{cases}\qquad \Ra\qquad g(s) = \begin{cases} 1-3s & s\in[0,1/2]\\ -1 + s &s\in[1/2, 1]\\ 0 &s\geq 1.\end{cases}
\]
Finally, since $\frac{d-3}2 =0$, we find that $(1-s^2)^\frac{d-3}2 \equiv 1$ and thus
\begin{align*}
f(re_1) &= \frac{\int_{-1}^1 g(rs) \ds} {\int_{-1}^1 1\ds} = \int_0^1 g(rs)\ds = \frac1r\int_0^r g(s)\ds = \frac1r \begin{cases} r-\frac32r^2 &0\leq r\leq 1/2\\ \frac{r^2}2 -r + \frac12 &1/2\leq r\leq 1\\ 0 &r\geq 1
\end{cases}\\
	& = \begin{cases} 1-\frac32r &0\leq r\leq 1/2\\ \frac{r}2 -1+ \frac1{2r} &1/2\leq r\leq 1\\ 0 &r\geq 1\end{cases}.
\end{align*}
In particular, we observe that $f\geq 0$ and that $f \in C^1(0,\infty)$. It is easy to see that the first derivative of $f$
\[
f'(r) = \frac12\left(-3\cdot\chi_{(0,1/2]}(r) + \left[1 - \frac1{r^2}\right]\cdot \chi_{(1/2,1]}(r)\right)
\]
is a continuous function, the second 
\[
f''(r) = \frac1{r^3}\cdot \chi_{(1/2, 1)}
\]
is a bounded and measurable function, and the third (distributional) derivative
\[
f'''(r) = 8\cdot \delta_{1/2} - \frac{3}{r^4}\cdot \mathcal L|_{(1/2,1)} - \delta_1
\]
is a finite measure, where $\delta_x$ denotes a Dirac delta located at the point $x$ and $\L_U$ denotes the one-dimensional Lebesgue measure of the open set $U$.

%As a measure of how much `steeper' or more concentrated $f$ is at the origin than the optimal Lipschitz interpolator
%\[
%h(x) := \max\{1-|x|, 0\} = \argmin\left\{ [h]_{Lip} : h(0) = 1\text{ and }h\equiv 0\text{ on }\R^3\setminus B_1(0)\right\},
%\]
%we compare
%\begin{align*}
%\|f\|_{L^1(\R^3)} &= \int_{\R^3}f(x) \dx = \frac{4\pi}3\int_0^1 r^2 f(r\,e_1)\dr = \frac{4\pi}3\cdot \frac1{32} = \frac{\pi}{24}\\
%\|h\|_{L^1(\R^3)} &= \int_{\R^3}h(x)\dx = \frac{4\pi}3\int_0^1r^2(1-r)\dr = \frac{4\pi}3 \cdot \left(\frac13 - \frac14\right) = \frac{\pi}9,
%\end{align*}
%so $\|f\|_{L^1}/\|h\|_{L^1} = \frac 9{24} = \frac{3}8 = 0.375$.
\end{example}

\begin{proof}[Proof of Theorem \ref{theorem main 3}]
The existence of a radial minimizer $f_{d,\eps}^*$ is proved as in Theorem \ref{theorem main 1}. By Lemma \ref{lemma radial barron functions}, we find that $f_{d,\eps}^* \in C^\frac{d-1}2(\R^d\setminus \{0\})$, and since $f_{d,\eps}^*$ is constant in a neighbourhood of the origin, we find that $f_{d,\eps}^*\in C^\frac{d-1}2(\R^d)$. The uniqueness follows as in Theorem \ref{theorem main 1} by considering the optimal measure $\mu$ on $[\eps,1]$ satisfying the moment conditions, using again Lemma \ref{lemma 1d reduction}. The main difference lies in the greater ability to uniformly approximate the function $f(s) = s$ by even polynomials on $[\eps,1]$ compared to $[0,1]$.

We claim the following: {\em Let $\eps>0$, $n\in\N$ and $\mu_n$ a measure on $[\eps,1]$ such that 
\[
\int_\eps^1 s\,\d\mu_n = 1, \qquad \int_\eps^1 s^{2k} \,\d\mu_n= 0\qquad \forall\ k= 0,\dots,n.
\]
Then for every $c<1$ there exists $N\in\N$ independent of $\eps$ such that 
\[
\|\mu_n\| \geq  c \,\frac{\eps^2\sqrt{\pi n}} {(1-\eps^2)^{n+1}}\]
if $n\geq N$.}

The claim is proved in Appendix \ref{appendix postponed}. Inserting the lower bound in \eqref{eq moments and norm 2}, the statement is proved. 
\end{proof}

\section{Applications}\label{section applications}

\subsection{Fitting values on a finite data set}

Let $(x_i, y_i)_{i=1}^N$ be a finite data set in $\R^d\times \R$. For each $i$, define $r_i = \min_{j\neq i} |x_j-x_i|$ to be the minimal distance between the point $x_i$ and the closest data point to it. Then
\[
f(x) = \sum_{i=1}^N y_i\, f_d^*\left(\frac{x-x_i}{r_i}\right)
\]
is a Barron function such that
\begin{enumerate}
\item $f(x_i) = y_i$ for all $i$ and
\item $\|f\|_\B \leq 2\gamma_{(d+1)/2} \sum_{i=1}^N \frac{|y_i|}{r_i}$.
\end{enumerate}

In most practical data sets, the minimum $\ell^2$-distance between data points is lower bounded as $\Omega(1)$ or even $\Omega(\sqrt d)$, meaning that the Barron norm only grows as $\sim dN$ or even $\sqrt{d}N$. Using the direct approximation theorem for Barron functions (Proposition \ref{proposition direct approximation}) in $L^2(\P_n)$ for $\P_n = \frac1n\sum_{i=1}^n \delta_{x_i}$, for every $m\in \N$ there exists a shallow neural network $f_m$ with $m$ neurons (and one constant shift) such that
\[
\frac1{n} \sum_{i=1}^n \big|f(x_i)-y_i\big|^2 \leq \frac{\|f\|_\B^2 \, \max_{|\nu|=1}\left\langle\sum_{i=1}^n (x_i-\bar x), \nu\right\rangle^2}m
\]
where $\bar x= \frac1n\sum_{i=1}^nx_i$. Often, the labels $y$ lie in a bounded set, at least with high probability. The projected and centered second moments may well be independent of the ambient dimension $d$, leading to a realistic, even somewhat pessimistic, expectation that
\[
L_\lambda(a,W,b) = \frac1n \sum_{i=1}^n \big|f_{(a,W,b)}(x_i) - y_i\big|^2 + \lambda\sum_{i=1}^m\big(a_i^2 + |w_i|_{\ell^2}^2\big) \lesssim \frac{d^2n^2}m + \lambda dn
\]
for data sets which do not heavily concentrate at a single point or exhibit heavy tail behavior. While the data can generally be fit exactly if $m> n$ (see e.g.\ \cite{llanas2006constructive} and the references therein) this estimate also controls the size of the weights of the neural network needed.

Similarly, this estimate can be used to bound the additional size of the Barron norm which is required to fit values $y_i' = y_i+\eps$, assuming that the Barron norm required to fit $y_i$ is already known.

\subsection{Mollification and density}

Since $f_d^*$ is a compactly supported, non-negative function, it can serve as a mollifier. Namely, for $\eps>0$ and $u\in L^1_{loc}(\R^d)$, denote
\[
\eta_\eps(z) = \frac{f_d^*(z/\eps)}{\|f_d^*\|_{L^1(\R^d)} \eps^d}, \qquad u_\eps(x) = (u*\eta_\eps)(x) = \int_{\R^d} u(z)\,\eta_\eps(x-z)\dz.
\]
It is well-known that $u_\eps \to u$ in $L^1(K)$ for any compact set $K\subseteq \R^d$ \cite[Lemma 4.22]{dobrowolski2010angewandte}. In many situations, rates can be obtained, either in the $L^1$-topology or a stronger topology, under the assumption that $u$ lies in a space of more regular functions (e.g.\ a H\"older or Sobolev space). We consider the following scenario:

Assume that $u\in X$, where $X\subseteq L^1(\R^d)$ is a space of functions $u:\R^d\to\R$ for which it is known that $\|u_\eps - u\|_{L^2(U)} \leq C\|u\|_X \eps^\alpha$ for a given domain $U\subseteq\R^d$ and some universal constants $C, \alpha>0$ (which may depend on $U$). If $u$ is naturally defined only on $U$ and not the entire space, extension theorems can often be used to extend $u$ in the same regularity class, see e.g.\ \cite[Chapter 6]{dobrowolski2010angewandte}. 

Note furthermore that $u_\eps$ is a continuous superposition of Barron functions in $x$. Since $\B_0$ is a Banach space, $u_\eps$ is a Barron function with norm at most
\[
[u_\eps]_\B \leq \|u\|_{L^1(\R^d)} \|\eta_\eps\|_\B = \frac{\|u\|_{L^1(\R^d)} \,[f_d^*]_\B}{\|f_d^*\|_{L^1(\R^d)}\,\eps^{d+1}}.
\]
In particular, due to the direct approximation theorem for Barron functions (Proposition \ref{proposition direct approximation}), there exists a neural network $f_m$ with one hidden layer, ReLU activation, $m$ neurons (and an affine shift) such that 
\[
\|f_m - u_\eps\|_{L^2(U)} \leq [u_\eps]_\B \,\meas(U)\,\diam(U)\,m^{-1/2}
\]
and thus
\[
\|f_m - u\|_{L^2(U)} \leq \|f_m - u_\eps\|_{L^2(U)} + \|u_\eps - u\|_{L^2(U)} \leq \frac{\|u\|_{L^1(\R^d)} \,[f_d^*]_\B\,\diam(U)\,\meas(U)}{\|f_d^*\|_{L^1(\R^d)}\,\eps^{d+1}m^{1/2}} + C\|u\|_X \eps^\alpha
\]
Balancing the scaling of terms $\eps^{-(d+1)}m^{-1/2} = \eps^\alpha$, we find that it is optimal to choose $\eps \sim m^{-\frac{1}{2(d+1+\alpha)}}$, which leads to an approximation order of $\eps^\alpha\sim m^{-\frac{\alpha}{2(d+1+\alpha)}}$. We note that not only the rate, but also the constants exhibit the curse of dimensionality. Observe that it is generally impossible to approximate functions in classical function spaces by functions of low norm from any function class in which the unit ball has low Rademacher complexity, so the curse of dimensionality cannot be avoided here \cite{approximationarticle}.

Since $f_d^*\leq 1$ and $f_d^*\equiv 0$ outside the unit ball, we find that $\|f_d^*\|_{L^1} \leq \omega_d %= \frac{\pi^{d/2}}{\Gamma(d/2+1)}
\sim \frac1{\sqrt{\pi d}} \left(\frac{2\pi e}{d}\right)^{d/2}$. The true $L^1$-norm is likely even much smaller, as $f_d^*$ appears to decay rapidly close to the unit sphere. Nevertheless, we find this an easy way to obtain an explicit rate with little effort.

\begin{example}
If $X$ is the space of Lipschitz-continuous functions on $\R^d$, then the approximation property holds as
\begin{align*}
\big|u_\eps(x) - u(x)\big| &= \left|\int_{\R^d} \big[u(z) - u(x)\big] \,\eta(x-z)\dz \right| \leq \int_{B_\eps(0)} \eta_\eps(z) |u(x+z) - u(x)|\dz\\
	&\leq [u]_{Lip} \frac{\int_0^\eps \eta_\eps(r)\,r^d\dr}{\int_0^\eps \eta_\eps(r)\,r^{d-1}\dr}\leq [u]_{Lip} \eps.
\end{align*}
The conditions above are therefore met with $\alpha=1$.
%If $\eta_\eps$ is concentrates sharply at the origin, then the norm may be smaller. We have not studied the behavior of $\eta_\eps$ derived from $f_d^*$ in this aspect. 
\end{example}

\subsection{Depth separation}

We have seen that any function which satisfies $f\equiv 1$ in $B_\eps(0)$ and $f\equiv 0$ outside of $B_1(0)$ has Barron semi-norm which is exponentially large in the dimension $d$ of the data space (for fixed $\eps \in (0,1)$).

By comparison, the function
\[
f(x) = \begin{cases} 1 &|x| \leq \eps \\ 1- \frac{|x|-\eps}{1-\eps} &\eps \leq |x|\leq 1\\ 0 &|x|\geq 1\end{cases}
\]
can be represented as the composition $f= f_1\circ f_2$ of two Barron functions $f_2:\R^d\to\R$ and $f_1:\R\to\R$
\[
 f_2(x) = \frac{|x|-\eps}{1-\eps}, \qquad f_1(z) = \max\big\{\{ 0, \min\{1-z, 1\}\big\}  = \sigma(1-z) - \sigma(-z)
\]
with norm
\[
[f_2]_\B = \frac1{1-\eps} \frac{\int_{S^{d-1}}1\,\d\H^{d-1}}{\int_{S^{d-1}}\sigma(\nu_1)\,\d\H^{d-1}_\nu} \sim \frac{2 \sqrt d}{1-\eps}, \qquad [f_1]_\B = 2.
\]
The second norm estimate be easily obtained by Proposition \ref{proposition one-dimensional}, whereas the second can be obtained as in the second step in the proof of \ref{lemma radial barron functions} in Appendix \ref{appendix postponed} -- see also \cite[Section 4]{barron_new}.

In particular, by the direct approximation theorem for Barron functions (Proposition \ref{proposition direct approximation}), it is possible to approximate $f_2$ with parameters whose magnitude does not exceed $C(1-\eps)d^{-1/2}$. When written as a neural network with two hidden layers, the initial linear layer of $f_1$ and terminal linear layer of $f_2$ are concatenated into a single linear map. Balancing the magnitude of coefficients equally over all layers, we find that the parameters scale only like $d^{1/6}(1-\eps)^{-1/3}$, so the weight decay regularizer grows as $d^{1/3}(1-\eps)^{-2/3}$. Observe that weight decay does not induce a norm for deeper ReLU networks due to the mismatch in homogeneities.

Theorem \ref{theorem main 3} thus serves to illustrate the following depth separation phenomenon: {\em A function $f:\R^d\to\R$ which takes values $1$ on $B_\eps(0)$ and $0$ on $\R^d\setminus B_1(0)$ is much easier to approximate by ReLU networks with two hidden layers than with one.} While depth separation phenomena are well established \cite{eldan2016power,telgarsky2016benefits}, this is a particularly easy criterion. The fact that compositions of Barron functions correspond to certain neural networks with two hidden layers has been observed e.g.\ in \cite{barron_new,parhi2022kinds}.

On the other hand, the result is a weaker version of a depth separations statement than others. We do not claim that the number of neurons required to approximate such a function $f$ to a certain accuracy grows exponentially in dimension, but rather that either the number of neurons or the magnitude of the parameters does. From a practical point of view, both are prohibitive.

\section{Finding optimal bump functions}\label{section examples}

In this section, we compute numerical approximations of the optimal bump functions which were constructed in Theorem \ref{theorem main 1} for different odd dimensions $d\in \N$ beyond the case $d=3$ considered in Example \ref{example d=3}. As previously, denote $n = \frac{d-1}2$, i.e.\ $d=2n+1$. For simplicity, we exploit that three tasks are equivalent: Approximating $|s|$ in $L^\infty(-1,1)$ by polynomials of degree at most $2n$ (or $2n+1$), approximating $\sqrt s$ in $L^\infty(0,1)$ by polynomials of degree at most $n$, and approximating $s$ in $L^\infty(0,1)$ by {\em even} polynomials of degree $2n$. We proceed in three steps:

\begin{enumerate}
\item Find the optimal approximation of $s\mapsto \sqrt{s}$ by polynomials of degree $n$ in $L^\infty(0,1)$, and find the $n+2$ points $t_0, \dots, t_{n+1}$ at which the error is maximal. Take the optimal points $s_i = t_i^2$ for the approximation of $f(s) = s$ by {\em even} polynomials. 

\item Solve the linear system \eqref{eq linear moment system} to obtain the measure $\mu= \sum_{i=0}^{n+1}\mu_i\delta_{s_i}$. Compute the piecewise linear function $g$ by $g'' = \mu$ in $(0,1]$, $g(0) = 1$ and $g'(0) = \mu_0$.

\item Obtain $f$ from $g$ by numerically integrating \eqref{eq f from g}.
\end{enumerate}

In our implementation, the first step is solved by the Remez algorithm \cite{trefethen2019approximation}:

\begin{enumerate}
\item[(i)] Initialize $s_0, \dots, s_{n+1}\subseteq[0,1]$, e.g.\ as equi-distant points such that $s_0=0$ and $s_1=1$.
\item[(ii)] Solve the system $\sum_{j=0}^n \alpha_j s^j_i = \sqrt{s} + (-1)^i e$, $0\leq i\leq n+1$ for the coefficients $\alpha_j$ and the equi-oscillation parameter $e$.
\item[(iii)] Update $s_0, \dots, s_{n+1}$ such that $s_0 = 0$, $s_{n+1}=1$ and for $i=1,\dots, n$, $s_i$ is a point at which the unsigned error function $\sqrt{s} - \sum_{j=0}^n \alpha_j s^j$ has a local extremum.
\item[(iv)] Iterate (ii) and (iii) until after the final update we have approximately reached equi-oscillating points of largest error:
\[
\frac {\max_{0\leq i\leq n+1}\left| \sqrt{s_i} - \sum_{j=0}^n \alpha_j s_i^j\right|}{\min_{0\leq i\leq n+1} \left|\sqrt{s_i} - \sum_{j=0}^n \alpha_j s_i^j\right|} < 1.001.
\]
\end{enumerate}

We solve the linear system in step 2 in the iteration scheme by LU factorization. The non-linear system is solved by two nested interval  constructions:
\begin{itemize}
\item Given $0=s_0<\dots < s_{n+1}=1$ such that $\sqrt{s_i} - \sum_{j=0}^n \alpha_j s_i^j = (-1)^ie$, we conclude that for all $i=0,\dots, n$, there exists $t_i \in (s_j, s_{j+1})$ such that 
\[
\sqrt{t_i} - \sum_{j=0}^n \alpha_j t_i^j = 0, \quad i=0,\dots, n
\]
by the Intermediate Value Theorem. In particular, the $n+1$ points $0 < t_0 < \dots < t_n<1$ are distinct and ordered. We approximate $t_i$ by the bisection method to accuracy $<10^{-12}$. Note that $e\neq 0$, since the approximating polynomial cannot match the objective function at $n+2$ points by the same argument as in the proof of Theorem \ref{theorem main 1}.

\item Given $t_0, \dots, t_n$, we find that there for $i = 1,\dots, n$ there exists $\xi_i \in (t_{i-1}, t_i)$ such that 
\[
\frac{d}{ds}\bigg|_{s=\xi_i} \left(\sqrt s - \sum_{j=0}^n\alpha_js^j\right) = 0, \qquad i=1, \dots, n
\]
by Rolle's Theorem. Again, we approximate $\xi_i$ by the bisection method to accuracy $<10^{-12}$. By construction, all $\xi_i$ are distinct. We update
\[
\{s_0,s_1,\dots, s_n, s_{n+1}\} \mapsto \{s_0, \xi_1, \dots, \xi_n, s_{n+1}\}.
\]
\end{itemize}

The nested interval construction is more numerically stable than Newton-Raphson iteration, as a Newton solver tends to  find the same $\xi_i$ multiple times starting at different roots $s_i, s_{i'}$ from the previous iteration. 

The linear step (2) is solved by LU factorization. The integral in (3) is evaluated using a composite Simpson rule and 1,001 integration points. A sample implementation of the algorithm can be found in a \href{https://colab.research.google.com/drive/1oFrzlAFDq73Ev1-MgMMnuua0LFt0f5gG?usp=sharing}{google colab notebook} at \cite{colab}. 

\begin{figure}
\begin{center}
\includegraphics[width= 0.32\textwidth]{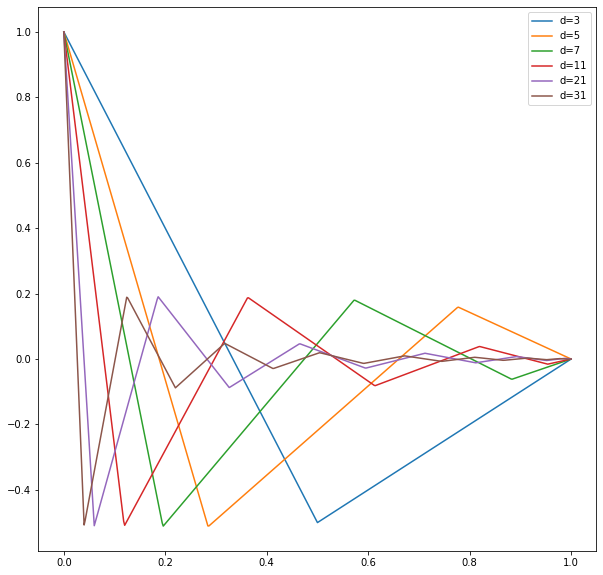}
\includegraphics[width= 0.32\textwidth]{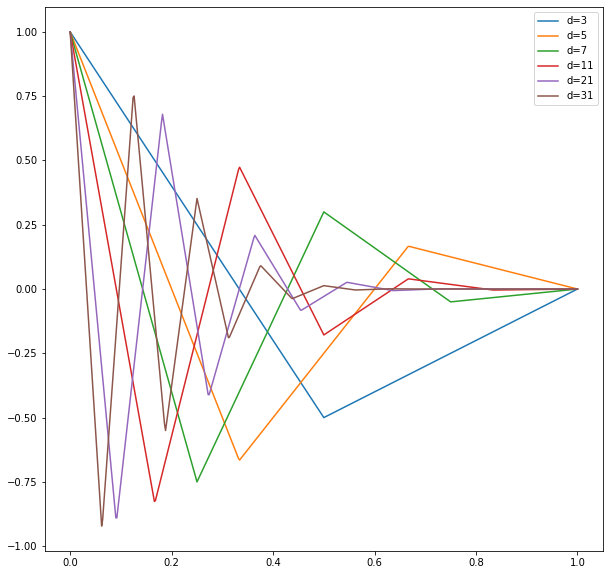}
\includegraphics[width= 0.32\textwidth]{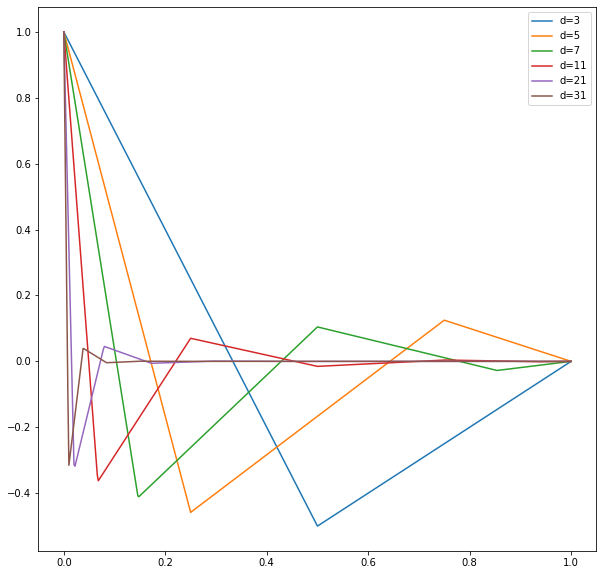}
\end{center}
\caption{\label{figure comparing g} We compute the piecewise linear function $g$ satisfying the moment conditions \eqref{eq linear moment system} for three different choices of $n+2 = \frac{d+3}2$ break points $s_i$: Optimal points found by the Remez algorithm (left), equi-distant points $s_i = i/(n+1)$ (middle) and the roots of Chebyshev polynomials of the second kind $s_i =1/2 + \cos\left(i\pi/(n+1)\right)/2$ (right).\\
For the optimal choice of $s_i$, we empirically observe that the collection of points $\{(s_i, g_d(s_i)) : d\in 2\N+1\}$ concentrates on a line $\ell_i$ parallel to the horizontal axis. For equi-distant nodes, the oscillations become larger as $d$ increases, whereas they become smaller for Chebyshev nodes.}
\end{figure}

We note that the linear system \eqref{eq linear moment system} can be solved for {\em any} choice of distinct points $0\leq s_0< \dots<s_{n+1} \leq 1$. To explore the importance of using the optimal points found using the Remez algorithm, we compare $g$ and $f$ for the optimal choice of sample points and other, more classic and explicit choices $s_i$ in Figures \ref{figure comparing g} and \ref{figure comparing f} respectively. The Barron norm grows slowly and linearly for optimal break points and faster than linearly for other explicit choices of break points, as can be seen in the rightmost image in Figure \ref{figure various plots}. In the left and middle plot of Figure \ref{figure various plots}, we also display the known profiles of Barron functions due to \cite{ongie2019function} as well as profiles of Barron functions which are constant in a neighbourhood of the origin.

For any choice of break points which include zero, the function $f$ is a `wizard's hat' function: Monotone decreasing, flat away from the origin, monotone decreasing and convex, non-smooth at the origin. It is thus qualitatively different from previously known radial profiles due to \cite{ongie2019function}.

\begin{figure}
\begin{center}
\includegraphics[width= 0.32\textwidth]{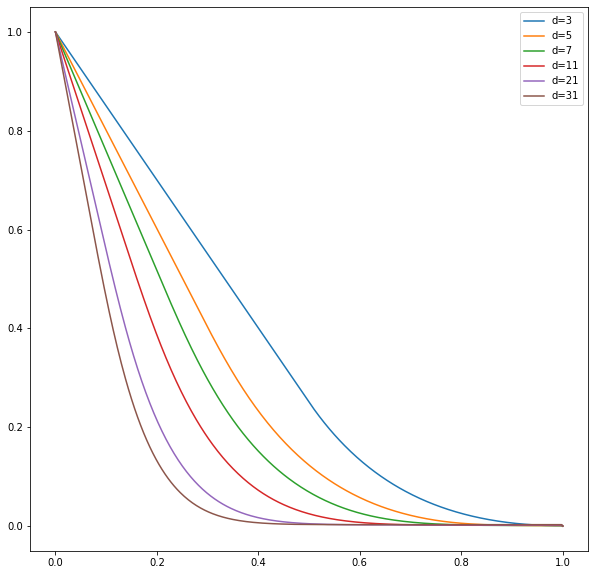}
\includegraphics[width= 0.32\textwidth]{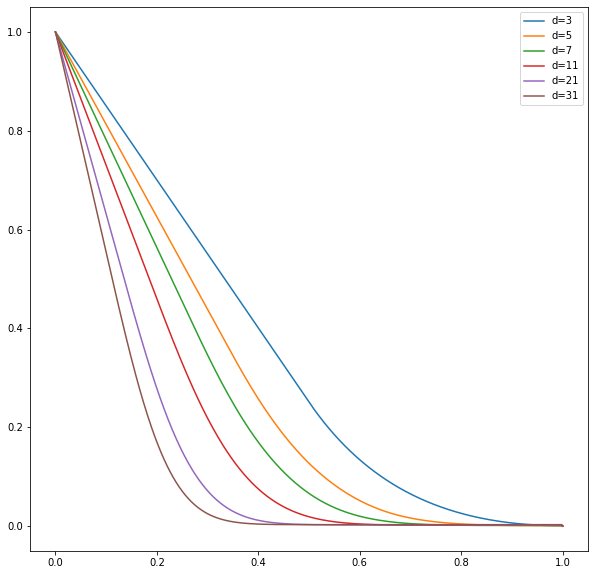}
\includegraphics[width= 0.32\textwidth]{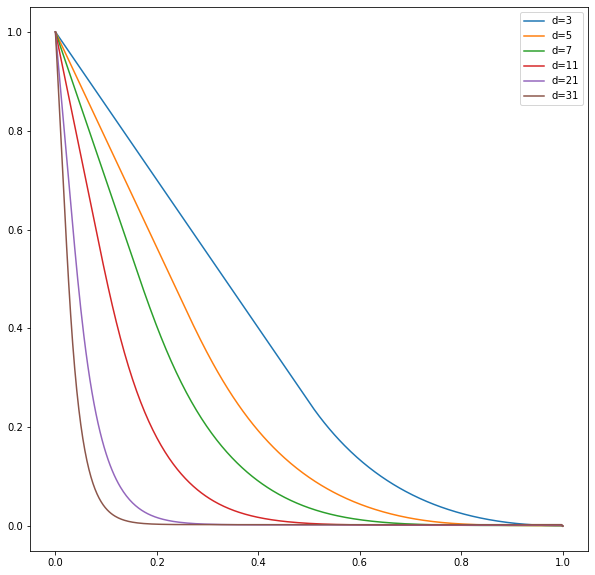}
\end{center}
\caption{\label{figure comparing f} We compare the functions $f$ computed by \eqref{eq f from g} for the piecewise linear functions $g$ in Figure \ref{figure comparing g} associated to three different choices of break points $s_i$: Optimal (left), equi-distant (middle), Chebyshev (right). The break points and curves agree for $d=3$ (blue curve) and are qualitatively similar for all $d\geq 3$, in particular non-negative, monotone-decreasing and convex. The curves are steeper at the origin in higher dimensions, most noticeably for Chebyshev nodes. The curve with optimal break points appears to make the slowest transition from $f=1$ to $f\approx 0$.}
\end{figure}

\begin{figure}
\begin{center}
\includegraphics[width= 0.32\textwidth]{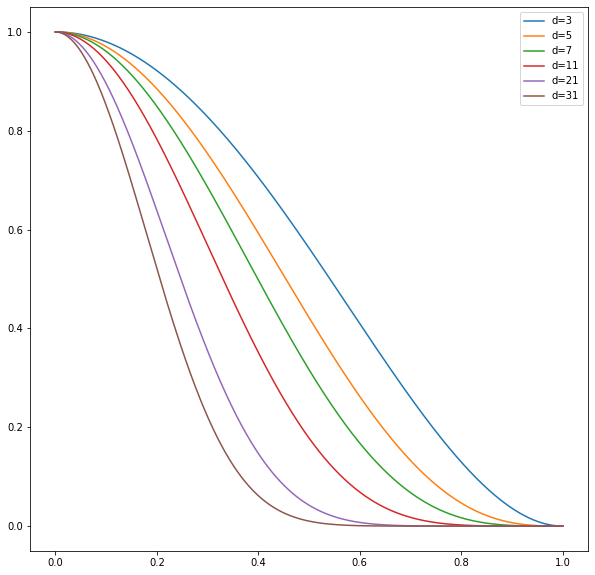}
\includegraphics[width= 0.32\textwidth]{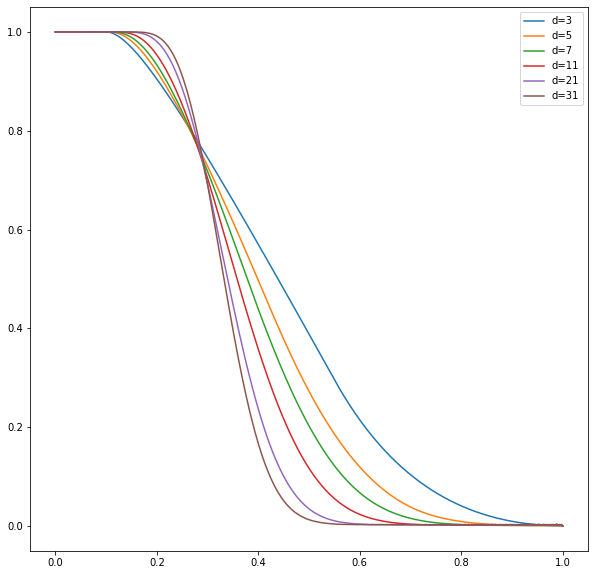}
\includegraphics[width= 0.32\textwidth]{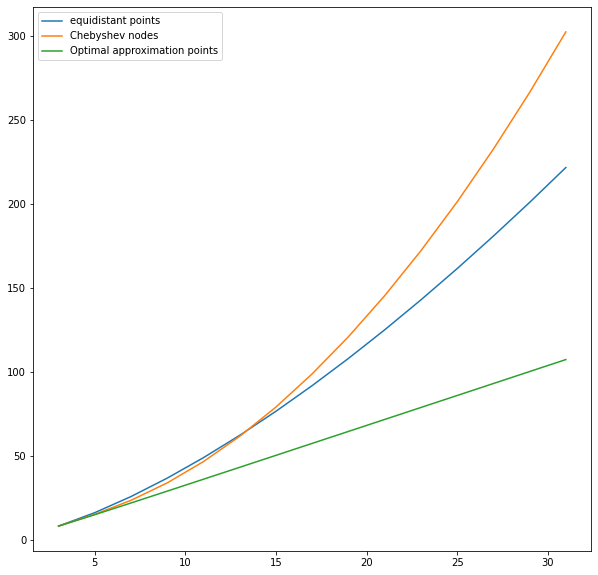}
\end{center}
\caption{\label{figure various plots} {\bf Left:} The radial profiles of the known Barron bump functions $f(|x|) = (1-|x|^2)^\frac{d-1}2$ of \cite{ongie2019function} are smooth at the origin and non-convex in the radial direction. They are thus geometrically distinct from the profiles associated to piecewise linear functions $g$ as depicted in Figure \ref{figure comparing f}. 
{\bf Middle:} The functions $f$ associated to piecewise linear functions $g$ with break points at equidistant points $s_i = 0.1 + 0.9\cdot i/(n+1)$ in $[0,1]$ are $C^1$-smooth, monotone decreasing and non-negative, but not convex in the radial direction.
{\bf Right:} The Barron semi-norm $\sum_{i=0}^{n+1}|\mu_i|$ of functions $f$ associated to $g$ with different break points $s_i$ grows slowest (and linearly) for optimal the optimal choice of points and fastest for break points at Chebyshev nodes. All growth rates are ostensibly polynomial of empirical degree $1.1$ (optimal points), 1.4 (equi-distant points) and 1.6 (Chebyshev nodes) as determined by least squares fitting. By comparison, the norm growth for the choice of equi-distant nodes in $[0.1,1]$ in the middle figure is exponentially large in $d$ and is not pictured for better readability of the plots.
}
\end{figure}

Additional empirical results relating to the optimal construction can be found in Appendix \ref{appendix plots}.

\section{Conclusion and Open Problems}\label{section conclusion}

We have provided an explicit construction for how neural networks optimally interpolate certain radially symmetric data with respect to a weight decay regularizer in the infinite parameter limit. While we do not prove that the optimal interpolant is radially symmetric, the radial average of all interpolants coincides with the solution constructed in this article. We show that its weight decay regularizer grows as $d$ and its Lipschitz constant grows at most as $\sqrt d$. In contrast, we identify a slight modification which necessitates exponential growth. A number of important questions remain open, even for shallow neural networks and the simple case of rotational symmetry. Deeper networks appear to be out of reach for our methodology.

{\bf Is the radially symmetric minimizer the only one?} A uniqueness statement would allow us to establish that regularized risk minimization does in fact lead to symmetry learning, at least in a toy example, and would allow us to prove stronger convergence results in the companion article \cite{wojtowytsch_radial_empirical}. We give further heuristic consideration to this question in Appendix \ref{appendix extensions}.

{\bf What happens if we modify the constraints?} For example, it is not clear from the proof of Theorem \ref{theorem main 3} whether the constraint $f\geq 1$ on $B_\eps(0)$ induces the curse of dimensionality as the constraint $f\equiv1$ on $B_\eps(0)$ does. Similarly, it may be interesting to study the case where the boundary condition $f\equiv 0$ is imposed on a shell $\{1\leq |x|\leq R\}$ rather than the entire exterior domain. We recover the problem studied in this article in the limit $R\to \infty$, whereas the optimal solution in the case $R=1$ would be $f(x) = 1-|x|$. Furthermore, a modified minimization problem is required to find optimal mollifiers:
\[
\text{Find } \tilde f_d^* \in \argmin_{f\in \widetilde\F} \frac{[f]_{\B}}{\|f\|_{L^1}}, \qquad \widetilde \F = \left\{f \in \B_0(\R^d) \cap C_c(\overline{B_1(0)}) : f\geq 0\right\}.
\]
It appears that subtle differences may make the difference between a solvable data fitting problem and one where we encounter the curse of dimensionality.

{\bf What more can we say about the optimal function $f_d^*$?} For example, we do not provide a lower bound on the Lipschitz constant of $f_d^*$, nor do we study the decay of $f_d^*(r)$ for fixed $r\in(0,1)$ or $\|f\|_{L^1(\R^d)}$ rigorously. We conjecture that both decay at least exponentially in $d$, and that both sequences are monotone in $d$. Limited evidence is provided in Appendix \ref{appendix extensions}.

Finally, the fact that the extrema of $g = g_n$ lie on straight lines parallel to the horizontal axis as we vary $n$ appears too specific to be random. It is not clear to us how to interpret this observation.

\section*{Acknowledgements}

The author would like to thank Jonathan Siegel and Rahul Parhi for inspiring conversations.

\bibliographystyle{../../alphaabbr}
\bibliography{../../NN_bibliography}

\newpage
\appendix

\section{Further plots}\label{appendix plots}

We note the following: If $g$ is a piecewise linear function with $n$ break points which satisfies the $n+2$ linear moment conditions \eqref{eq moment conditions}, then it also satisfies the same linear moment conditions for any $m \leq n$. In particular, the function
\[\showlabel\label{eq reduced dimension break points}
f_{m,n}:\R^{2m+1}\to \R, \qquad f_{m,n}(x) = \frac{\int_{-1}^1 g\big(|x|s\big)\,(1-s^2)^\frac{m-3}2 \d s}{\int_{-1}^1 (1-s^2)^\frac{m-3}2 \d s}
\]
is a Barron function such that $f_{m,n}(0) = 1$ and $f_{m,n}\equiv 0$ on $\R^k\setminus B_1(0)$ for $k\leq n$. We plot $f_{m,n}$ for $m=1$ (i.e.\ $2m+1=3$) and various choices of $n$ and various choices of break points in Figure \ref{figure more points, 3d construction}. In Figure \ref{figure more points, varying dimension} we fix $n= 10$ instead and consider the influence of varying $m$.

The larger the discrepancy between $m$ and $n$, the more oscillatory the function $f_{m,n}$ is. This is reminiscent of observations from Step 4 in the proof of Theorem \ref{theorem main 1}.

In Figure \ref{figure decay of f_d}, we numerically investigate the decay of $f_d^*$ as $d$ varies, both pointwise in $r$ and integrated. Since $f_d^*$ is $\sqrt d$-Lipschitz and $f_d^*(0)=1$, we see that the one-dimensional integral $\int_0^1 f_d^*(r)\dr$ is bounded from below by $\Omega(d^{-1/2})$. This indeed appears to be the dominant term, and for fixed $r>0$, we observe empirically that $f_d^*(r)$ decays to zero exponentially in $d$.

\begin{figure}
\begin{center}
\includegraphics[width= 0.32\textwidth]{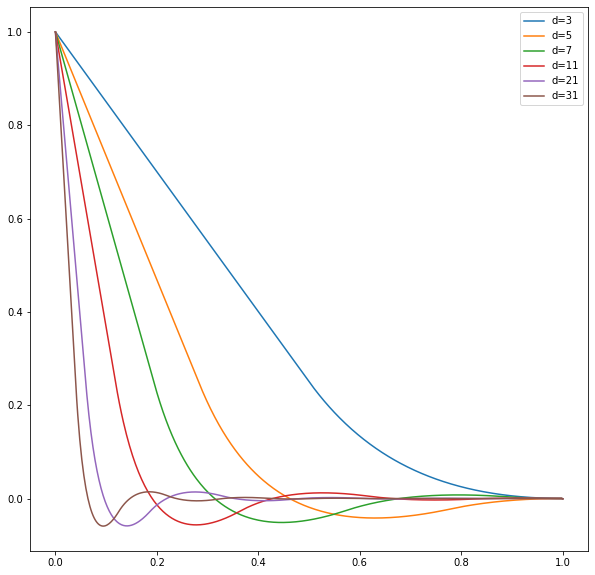}
\includegraphics[width= 0.32\textwidth]{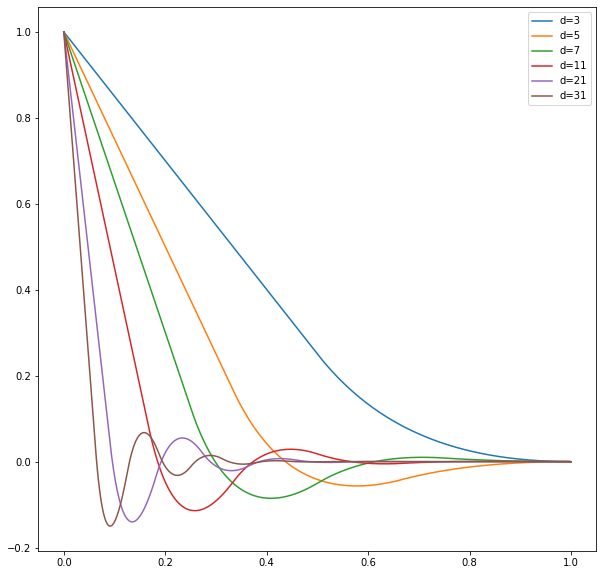}
\includegraphics[width= 0.32\textwidth]{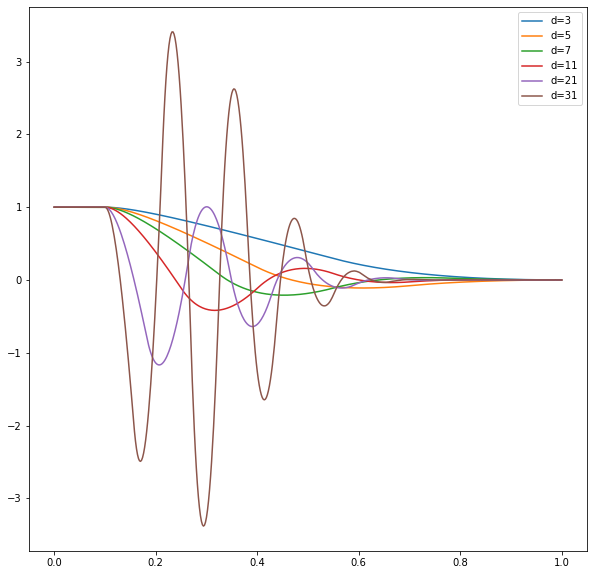}
\end{center}
\caption{\label{figure more points, 3d construction} We plot the radial profile of $f_{1,n}:\R^3\to\R$ as in \eqref{eq reduced dimension break points} for various choices of $n= \frac{d-1}2$ break points. The points are chosen optimally (for dimension $d=2n+1$) on the left, equi-distant in $[0,1]$ in the middle plot and equidistant in $[0.1, 1]$ on the right. Notably, the functions are neither monotone nor non-negative if $d>3$, and the number of local extrema increases as $n$ grows.}
\end{figure}

\begin{figure}
\begin{center}
\includegraphics[width= 0.32\textwidth]{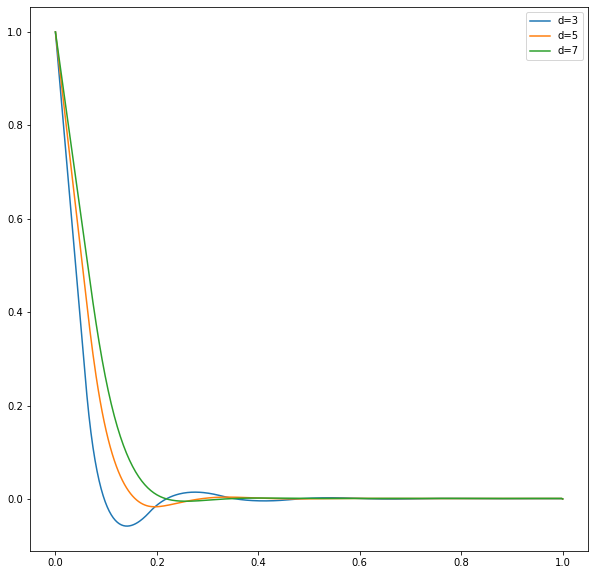}
\includegraphics[width= 0.32\textwidth]{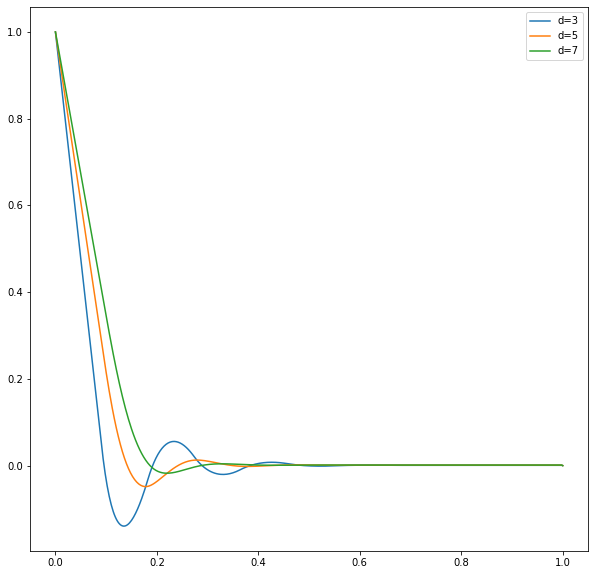}
\includegraphics[width= 0.32\textwidth]{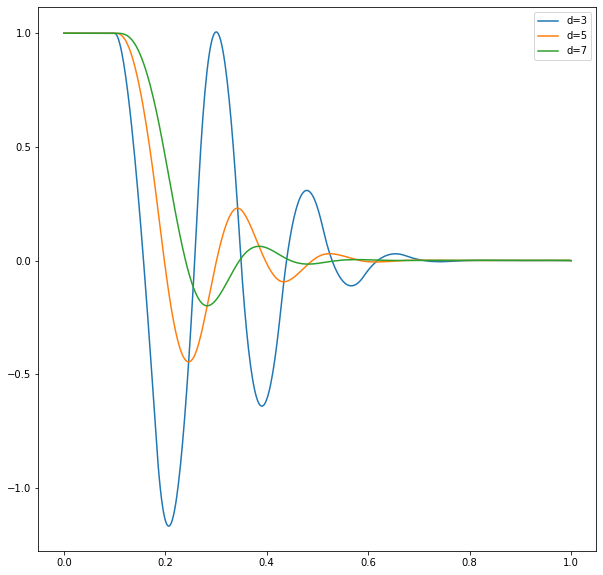}
\end{center}
\caption{\label{figure more points, varying dimension} We plot $f_{m,n}$ as in \eqref{eq reduced dimension break points} corresponding to low dimension $d\in \{3,5,7\}$, $m= \frac{d-1}2$ and $n= 10$. and various choices of break points: Optimal for $n=10$ (left), equidistant in $[0,1]$ (middle) and equi-distant in $[0.1, 1]$ (right). The radial profiles of the Barron functions are neither monotone nor non-negative. The number of local extrema of the profiles is larger if the dimension $d$ is small compared to the number of break points. The oscillations are smallest for the optimal choice of break points and largest for break points which are bounded away from the origin.}
\end{figure}

\begin{figure}
\begin{center}
\includegraphics[width= 0.32\textwidth]{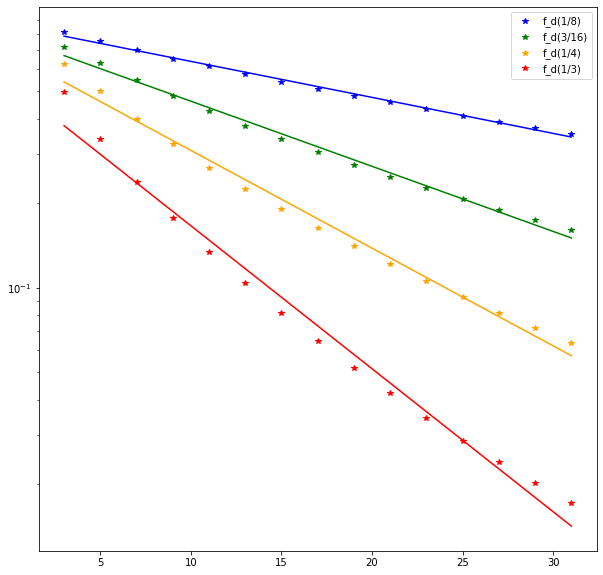}
\includegraphics[width= 0.32\textwidth]{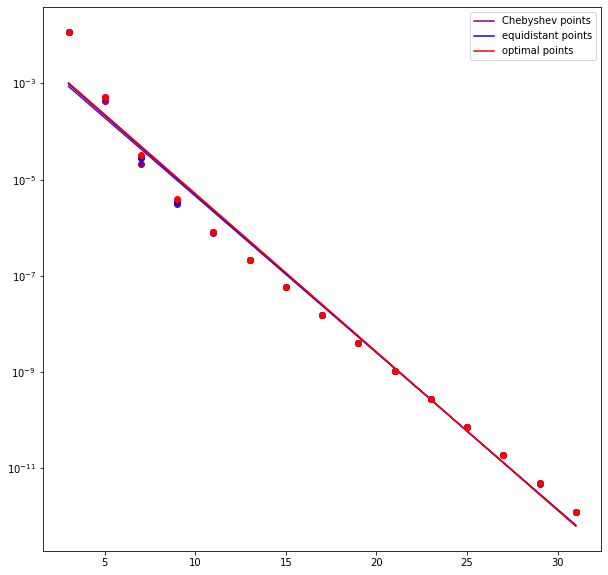}
\includegraphics[width= 0.32\textwidth]{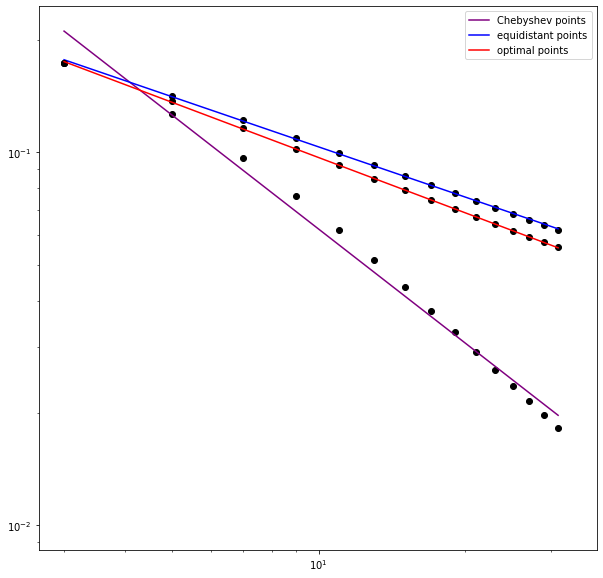}
\end{center}
\caption{\label{figure decay of f_d} {\bf Left:} We plot $f_d^*(x)$ as a function of $d$ for fixed $x$ in a logarithmic scale together with $\exp(\alpha(x) d+ \beta(x))$, where $\alpha, \beta$ are chosen depending on $x$ as the least squares fit for the function $\log(f_d^*(x))$. The graphs suggest that the decay is exponential in $d$ and thus comparable to explicit solutions $\tilde f_d(x) = (1-|x|^2)^{\frac{d+3}2}$ of \cite{ongie2019function}.
{\bf Middle:} We graphically compare the decay of the normalized $d$-dimensional integral of $\avint_{B_1(0)} f_d(x)\dx$ for different choices of break points. Despite graphical differences around the origin, the values of the integrals are very similar and decay roughly as $\exp(-4.7 - 0.75* d)$. In dimension three, all three functions $f_d$ coincide, while their difference close to the origin becomes negligible in high dimension, where almost all measure concentrates by the boundary of the unit ball. 
{\bf Right:} We graphically compare the decay of the $1$-dimensional integral $\int_0^1 f_d(r\cdot e_1)\dr$ of the function $f_d$ associated to piecewise linear $g$ with $n = \frac{d-1}2$ break points in $(0,1)$ for different choices of break points. The integral  empirically decays as $0.3\cdot d^{-0.49}$ for optimal points, like $0.29\cdot d^{-0.45}$ for equidistant points and like $0.64\cdot d^{-1.02}$ for Chebyshev points. The order of decay was  established by a least squares regression.
}
\end{figure}

\section{Postponed proofs}\label{appendix postponed}

Recall the co-area formula, which allows us to integrate over a Riemannian manifold $M$ by `slicing' the domain into the level sets of a function $\phi:M\to N$, where $N$ is another Riemannian manifold \cite[Theorem 13.4.2]{burago2013geometric}. In the case of slicing the sphere into level sets of a coordinate projection $\phi(x) = x_1$, the formula reads as
\[
\int_{S^{d-1}} f(x)\,\d\H^{d-1}_x = \int_{-1}^1\left( \int_{S^{d-1}\cap \{x_1 = s\}}  f(x) \,\d\H^{d-2}_x\right) (1-s^2)^{-1/2}\ds
\]
since $1- x_1^2 = |\nabla^\parallel \phi|^2$ is the modulus of the tangential gradient of $\phi$, which measures volume distortions. This can be considered a curvinlinear version of Fubini's theorem. If $f$ only depends on $x_1$, the formula further simplifies to 
\begin{align*}\showlabel\label{eq co-area formula}
\int_{S^{d-1}} f(x)\,\d\H^{d-1}_x %&= \int_{-1}^1 \int_{S^{d-1}\cap \{x_1 = s\}}  f(s,0,\dots,0)\,\d\H^{d-2} (1-s^2)^{-1/2}\ds\\
	&= (d-1)\omega_{d-1} \int_{-1}^1  f(s,0,\dots,0)\, (1-s^2)^{(d-3)/2}\ds
\end{align*}
since $S^{d-1}\cap \{x_1=s\}$ is a $d-2$-dimensional Euclidean sphere of radius $\sqrt{1-s^2}$. Here $\omega_{d-1}$ denotes the volume of the $d-1$-dimensional unit ball and $(d-1)\omega_{d-1}$ the volume of the $d-2$-dimensional Euclidean unit sphere.

\begin{proof}[Proof of Lemma \ref{lemma radial barron functions}]
{\bf Step 1. Symmetrization.} Let $f:\R^d\to\R$ be a radially symmetric Barron function. Then in particular $f(x)$ is the same as the average over $f(Ox)$ for $O$ in $SO(d)$ and the average is taken with respect to the Haar measure $H$ (which coincides with the $\frac{d(d-1)}2$-dimensional Hausdorff measure on $SO(d)$ with respect to the Frobenius norm), i.e.
\begin{align*}
f(x) &= \avint_{SO(d)} f(Ox)\,\d H_O\\
	&= \avint_{SO(d)} \int_{\R^{d+2}} \sigma(w^TOx+b)\,\d\mu_{(w,b)}\,\d H_O\\
	&= \int_{\R^{d+2}} \avint_{SO(d)} \sigma\big((O^Tw)^Tx+b)\,\d H_O\,\d\mu_{(w,b)}\\
	&= \int_{\R^{d+2}}|w| \avint_{S^{d-1}} \sigma\left(\nu^Tx+\frac {b}{|w|}\right)\,\d \H^{d-1}_\nu\,\d\mu_{(w,b)}
\end{align*}
since for any $w\in S^{d-1}$, the map $SO(d) \to S^{d-1}$, $O\mapsto Ow$ pushes the Haar measure forward to the uniform distribution on $S^{d-1}$. Thus $f$ can be written as a continuous linear combination of the elementary radially symmetric Barron functions
\[
f_b(x) = \avint_{S^{d-1}} \sigma\big(\nu^Tx-b)\,\d \H^{d-1}_\nu \quad\text{and }f_\infty(x) \equiv 1.
\]
On the other hand, every function of this type is a radially symmetric Barron function. Finally, we note that 
\begin{align*}
f_b(x) - f_{-b}(x) &= \avint_{S^{d-1}} \sigma\big(\nu^Tx-b) -\sigma\big(\nu^Tx+b) \,\d \H^{d-1}_{\nu}\\
	&= \avint_{S^{d-1}} \sigma\big(-\nu^Tx-b) -\sigma\big(\nu^Tx+b) \,\d\H^{d-1}_\nu\\
	&= \avint_{S^{d-1}} \nu^Tx+b \,\d\H^{d-1}_\nu\\
	&= b,
\end{align*}
since the uniform distribution is invariant under the substitution $\nu\mapsto-\nu$ in the first term.
In particular, every radially symmetric Barron function can be written as
\[\showlabel\label{eq radially symmetric barron function}
f(x) = f(0) + \int_{[0,\infty)} f_{b}(x) \,\d\mu_b
\]
for some measure $\mu$ on $[0,\infty)$ and $[f]_\B = \|\mu\|_{TV}$, since $f_{b}(0) = 0$ for any $b>0$.

{\bf Step 2. Gradient bound.} We note that $[f_b]_\B=1$ for any $b$ by definition and
\begin{align*}
\nabla f_{b}(x) &= \avint_{S^{d-1}} \sigma'\big(\nu^Tx-b) \,\nu\,\d\H^{d-1}_\nu
	= \avint_{S^{d-1}} 1_{\{\nu^Tx>b\}} \nu\,\d\H^{d-1}_\nu.
\end{align*}
Due to radial symmetry, the gradient points in direction $x$, i.e.\
\begin{align*}
\nabla f_{b}(x) &= \frac{\int_{S^{d-1}} 1_{\{\nu_1> b/|x|\}} \nu_1\,\d\H^{d-1}_\nu} {\int_{S^{d-1}} 1\,\d\H^{d-1}_\nu}\,\frac{x}{|x|}
\end{align*}
The gradient is largest as $|x|\to \infty$, and 
\[
\sup_{x\in\R^d} |\nabla f_{b}(x)| = \frac{\int_{S^{d-1}} 1_{\{\nu_1>0\}} \nu_1\,\d\H^{d-1}_\nu} {\int_{S^{d-1}} 1\,\d\H^{d-1}_\nu} = \frac{\int_{S^{d-1}} 1_{\{\nu_1>0\}} \nu_1\,\d\H^{d-1}_\nu} {2\int_{S^{d-1}}  1_{\{\nu_1>0\}}\,\d\H^{d-1}_\nu} = \frac{\int_0^1 s(1-s^2)^\frac{d-3}2\ds}{2\int_0^1(1-s^2)^\frac{d-3}2\ds}
\]
independently of $b$. In the last step, we used the coarea formula \eqref{eq co-area formula}. It is now possible to evaluate the gradient
\[
\frac{\int_0^1 s(1-s^2)^\frac{d-3}2\ds}{2\int_0^1(1-s^2)^\frac{d-3}2\ds} = \frac{\frac1{d-1}} { \sqrt\pi \frac{\Gamma\big((d-1)/2\big)}{\Gamma(d/2)}} = \frac{\Gamma(d/2)} {\sqrt\pi\,(d-1)\,\Gamma\big((d-1)/2\big)} \sim \frac{1}{\sqrt{2\pi d}}
\]
in the sense that 
\[
\lim_{d\to\infty} \sqrt d\frac{\int_0^1 s(1-s^2)^\frac{d-3}2\ds}{2\int_0^1(1-s^2)^\frac{d-3}2\ds} = \frac1{\sqrt{2\pi}}.
\]
Consequently, for a general radially symmetric Barron function as in \eqref{eq radially symmetric barron function} and sufficiently large $d\in\N$, we find that 
\[
[f]_{Lip} = \sup_{x\in\R^d} \big|\nabla f(x)\big| \leq \int_{[0,\infty)} \|\nabla f_b\|_{L^\infty}\, \d|\mu|_b \leq \frac{1+\eps}{\sqrt{2\pi d}} \|\mu\|_{TV} = \frac{1+\eps}{\sqrt{2\pi d}} [f]_\B.
\]

{\bf Step 3. Higher regularity.} By Corollary \ref{corollary structure barron}, any radially symmetric Barron function is $C^1$-smooth except at the origin. This establishes the claim in the case $d=3$.

Note that if $f:(0,\infty)\to \R$ is $C^k$-smooth, then the same is true for $F:\R^d\to\R$ given by $F(x) = f(|x|)$ by the chain rule and product rule. It thus suffices to analyze the radial profile $f$ of $F$. In the following, we will denote both functions as $f$ by a slight abuse of notation. Consider the radial profile of the function
\[
f_{b}(r) = c_d {\int_0^1 \sigma(sr -b)\,(1-s^2)^\frac{d-3}2\ds}, \qquad c_d = \frac1{\int_{-1}^1(1-s^2)^\frac{d-3}2\ds}
\]
for $b>0$. We can compute the first two derivatives of $f_{b}$ by exchanging differentiation and integration
\begin{align*}
f_{b}'(r) &= c_d{\int_0^1 \sigma'(sr -b)\,s(1-s^2)^\frac{d-3}2\ds}\\
f_{b}''(r) &= c_d{\int_0^1 \sigma''(sr -b)\,s^2(1-s^2)^\frac{d-3}2\ds}{} \quad = \frac{c_d}r\big(b/r\big)^2 \max\left\{1- \big(b/r\big)^2,0\right\}^\frac{d-3}2,
\end{align*}
where the second formula must be justified by approximation, as the derivative $\frac{d^2}{dr^2}\sigma(sr-b) = \frac1r\cdot \delta_{b/r}$ (considered as a `function' of $s$) is not regular. For $b>0$, it is easy to see that  $f_{b}\equiv 0$ is $C^\infty$-smooth in $[0,b)$, and as a polynomial in $1/r$ also $C^\infty$-smooth on $(b,\infty)$. Clearly $f_b''$ and all its derivatives vanish at infinity. If $d=3$, $f_b''$ is continuous except at $r=b$, where it has a jump discontinuity. If $d\geq 5$, the function
\[
f_b''(r) = c_d\,r^{-d} b^2(r^2-b^2)^\frac{d-3}2 = c_d b^2\,r^{-d}(r-b)^\frac{d-3}2(r+b)^\frac{d-3}2
\]
vanishes as $(r-b)^{\frac{d-3}2}$ at $r=b$ and thus has $\frac{d-5}2$ additional derivatives which vanish at $r=b$. We find $f_b \in C^\frac{d-1}2$ for any odd dimension $d$. The $\frac{d+1}2$-th derivative of $f_b$ is bounded and continuous except at $r=b$, and the $\frac{d+3}2$-th derivative of $f_b''$ is a finite measure associated to the regular part of the derivative in $(b,\infty)$ and the jump at $r=b$.

It remains to show that a general radial Barron function
\[
f(r)= f(0) + \int_{[0,\infty)} f_b(r)\,\d\mu_b
\]
has the same regularity as its components $f_b$, at least away from the origin. To simplify the presentation, we focus on the case $d\geq 5$. Let $\eps>0$ and observe that
\[
f(r)= f(0) + \mu(\{0\})\, f_0(r) + \int_{(0,\eps]} f_b(r)\,\d\mu_b + \int_{(\eps, \infty)} f_b(r)\,\d\mu_b.
\]
Clearly, the affine linear component $ f(0) + \mu(\{0\})\, f_0(r)$ is $C^\infty$-smooth except at the origin. Secondly, we note that for any $b>0$, the identity $f_b(r) = b\,f_1\big(r/b\big)$ holds. In particular,
\[
\frac{d^k}{dr^k} \int_{(\eps, \infty)} f_b(r)\,\d\mu_b = \int_{(\eps, \infty)} b^{1-k} \,f_1^{(k)}\left(\frac rb\right)\,\d\mu_b  
\]
where the integrals converge uniformly for $k\leq \frac{d-1}2$ due to the $L^\infty$-bound on the $k+1$-th derivative of $f_b$. Similarly, the $\frac{d+1}2$-th derivative converges in $L^p$ for all $p<\infty$ due to the bound on the measure-valued $\frac{d+3}2$-th derivative. Finally, for $k= \frac{d+3}2$, the integral converges weakly in the sense of Radon measures, i.e.\ in the weak-* sense, when we consider the space of (Radon) measures as dual to the space of continuous functions.

For the first integral, we prove convergence assuming that $r\geq \eps$. Note that $f_b''(r) = r^{-1} \,P(b/r)$ for some polynomial $P$ and $r\geq \eps\geq b$. By induction we see that $f_b^{(k)}(r) = r^{1-k}P_k(b/r)$ for all $k\geq 2$, where $P_k$ is another polynomial. This is easily seen since
\[
\frac{d}{dr} \big[ r^{1-k} P_k(b/r)\big] = (1-k)\,r^{-k} \,P_k(b/r) + r^{1-k}\,P_k'(b/r)\,\left(-\frac{b}{r^2}\right) = r^{-k} \left((1-k)\,P_k(b/r) - \frac br\,P_k'(b/r)\right).
\]
Hence, as before,
\[
\frac{d^k}{dr^k} \int_{(0,\eps]} f_b(r)\,\d\mu_b = \int_{(0,\eps]} r^{1-k} P_k(b/r)\,\d\mu_ b = r^{1-k}\int_{(0,\eps]} P_k(b/r)\,\d\mu_b.
\]
The integral converges since $b/r \in [0,1]$ and $P_k$ is a continuous function.
\end{proof}

\begin{proof}[Proof of Lemma \ref{lemma 1d reduction}]
{\bf First claim.} 
Assume that $g:\R\to\R$ is a function with the properties outlined above and $f$ is defined by \eqref{eq f from g}. Then $f$ is radially symmetric by definition and
\[
f(0) = \frac{\int_{-1}^1(1-s^2)^\frac{d-3}2\,g(0)\ds}{\int_{-1}^1(1-s^2)^\frac{d-3}2\ds} = 1.
\]
Furthermore, if $r= |x|\geq 1$ and $\tilde c_d:= \int_{-1}^1(1-s^2)^\frac{d-3}2\,\ds$, then
\begin{align*}
\tilde c_d\,f(x) &= \int_{-1}^1 (1-s^2)^\frac{d-3}2\,g(rs)\ds\\
	&= r^{2-d} \int_{-1}^1 (r^2-(rs)^2)^\frac{d-3}2\,g(rs)\,r\ds\\
	&= r^{2-d} \int_{-r}^r  (r^2-z^2)^\frac{d-3}2\,g(z)\dz\\
	&=  r^{2-d} \int_{-1}^1  (r^2-z^2)^\frac{d-3}2\,g(z)\dz
\end{align*}
since $g\equiv 0$ outside of $(-1,1)$. The integral vanishes by \eqref{eq orthogonality conditions} if $d\geq 3$ is odd, since $(r^2-z^2)^\frac{d-3}2$ is an even polynomial of degree at most $d-3$ for all $r$. 

It remains to show that $f$ is a Barron function.
 Take any measure $\mu$ such that 
\[
g(x) = \int_{-\infty}^\infty \sigma(x+b)\,\d\mu_b
\]
as in Proposition \ref{proposition one-dimensional}
and compute that
\begin{align*}
\int_{-\infty}^\infty \avint_{S^{d-1}} \sigma(\nu^Tx+b)\,\d\H^{d-1}\,\d\mu_b \showlabel\label{eq f is barron}
	&=  \avint_{S^{d-1}} \int_{-\infty}^\infty \sigma(\nu^Tx+b)\,\d\mu_b\,\d\H^{d-1}_\nu\\
	&= \avint_{S^{d-1}} g(\nu^Tx)\,\d\H^{d-1}_\nu\\
	&= \frac{|S^{d-2}|}{|S^{d-1}|} \int_{-1}^1 g(s)\,(1-s^2)^\frac{d-3}2\ds \quad= f(x)
\end{align*}
by the co-area formula \eqref{eq co-area formula}.
The fact that the normalizing constant is exactly
\[
\tilde c_d = \frac{|S^{d-2}|}{|S^{d-1}|} = \int_{-1}^1 (1-s^2)^\frac{d-3}2\ds
\]
can be justified by the same co-area integration. Finally, we note that the left hand side of \eqref{eq f is barron} is clearly a radially symmetric Barron function satisfying $[f]_\B \leq \|\mu\|_{TV}$. Taking the infimum over all $\mu$ representing $g$, we find that $[f]_{\B(\R^d)} \leq [g]_{\B(\R)}$.

{\bf Second claim.} Assume on the other hand that $f$ is a radially symmetric Barron function. If we denote by $\overline\mu$ the Haar measure on the special orthogonal group $SO(d)$, then due to radial symmetry
\begin{align*}
f(x) &= \int_{SO(d)} f(Ox) \,\d\overline\mu_O\\
	&= f(0)+  \int_{SO(d)} \int_{\R^{d+1}} \sigma(w^TOx+b)\,\d\mu_{(w,b)} \,\d\overline\mu_O\\
	&= f(0)+   \int_{\R^{d+1}} \avint_{SO(d)} \sigma\big((O^Tw)^Tx+b\big) \,\d\overline\mu_O\,\d\mu_{(w,b)}\\
	&= f(0)+  \int_{\R^{d+1}} |w| \avint_{S^{d-1}} \sigma\left(\nu^Tx+\frac{b}{|w|}\right)\d\H^{d-1}_\nu\d\mu_{(w,b)}\\
	&= f(0)+   \int_{-\infty}^\infty \avint_{S^{d-1}} \sigma\left(\nu^Tx+b'\right)\,\d\H^{d-1}_\nu\,\d\hat\mu_{b'}
\end{align*}
where $\hat \mu= \Phi_\sharp(|w|\cdot \mu)$ for the map
\[
\Phi :\R^{d+1}\to \R, \qquad \Phi(w,b) = \frac{b}{|w|}.
\]
It is now possible to reverse the calculations from Step 1 by setting 
\[
g:\R\to\R, \qquad g(x) = \int_\R \sigma(x - b)\,\d\hat\mu_b.
\]
Taking the infimum over $\mu$ representing $f$, we find that $[g]_{\B(\R)}\leq [f]_{\B(\R^d)}$. Clearly, both $s\mapsto g(s)$ and $s\mapsto g(-s)$ induce the same function $f$ by \eqref{eq f from g} due to symmetry, and so does the even representative $s\mapsto \big(g(s) + g(-s)\big)/2$.

It remains to show that $g\equiv 0$ outside of $(-1,1)$ and that the moment conditions \eqref{eq moment conditions} hold. Assuming that $g\equiv 0$ outside $(-1,1)$, we find that 
\[
0 =\int_{-1}^1 (1-s^2)^\frac{d-3}2 g(rs)\ds = r^{2-d}\int_{-1}^1 \big(r^2- (rs)^2\big)^\frac{d-3}2 g(rs)\,r\ds = r^{2-d}\int_{-1}^1 (r^2-z^2)^\frac{d-3}2\,g(z)\dz
\]
for all $r\geq 1$, as the integral over $(-r,-1) \cup(1,r)$ vanishes. The moment conditions follow easily as
\[
\frac{d^k}{dr^k} \int_{-1}^1 (r^2-z^2)^\frac{d-3}2 \,g(z)\dz = \frac{d^k}{dr^k}\sum_{j=0}^{(d-3)/2}\binom{(d-3)/2}{j} r^{d-3-2j}\int_{-1}^1 g(z) \,z^{2j} \dz\equiv 0
\]
for $r\in [1,\infty)$ and $k\geq 1$. Taking $k=\frac{d-3}2$ derivatives, we find that $g$ is $L^2$-orthogonal to $z^0$. Lowering the order of the derivative inductively, we find that $g$ is $L^2$-orthogonal to all even polynomials of degree at most $d-3$.

 Thus we only need to show that $g\equiv 0$ outside $(-1,1)$. First consider the case $d=3$, i.e. $n=0$ and thus
\[
f(r) = \int_{-1}^1 g\big(r s\big)\ds.
\]
Then for $r\geq 1$, we have
\[
0 = f(r) %=\int_{-1}^1 g\big(|x| s\big)\ds
 = \frac1{r} \int_{-1}^1 g\big(r s\big) \,r\ds = \frac1{r} \int_{-1}^1 g(z)\dz + \frac1{r} \int_1^{r} g(z) + g(-z)\dz = \frac1{r} \int_1^{r} g(z) + g(-z)\dz
\]
since $\int_{-1}^1 g(s)\ds = f(1) = 0$. As $g$ is an even function, we conclude that $\int_1^r g(s)\ds =0$ for all $r\geq 1$ and thus $g(s) =0$ for all $s>1$. We now proceed inductively: Assume that $n\geq 1$ is such that
\[
r^{-(n+1)}\int_{-r}^r g(z) (r^2-z^2)^n\dz =  r^{-(n+1)}\int_{-1}^1 g(rs) (r^2-(rs)^2)^n r\ds= \int_{-1}^1 g(rs) (1-s^2)^n\ds = 0
\]
for all $r\geq 1$. Then also
\[
0 \equiv \int_{-r}^r g(z) (r^2-z^2)^n\dz\quad \Ra\quad 0 \equiv \frac{d}{dr} \int_{-r}^r g(z) (r^2-z^2)^n\dz = 2r \int_{-r}^r g(z) (r^2-z^2)^{n-1}\dz
\]
since the boundary term vanishes for $n\geq 1$. In particular, we conclude that 
\[
\int_{-1}^1 g(rs) (1-s^2)^{n}\ds \equiv 0 \quad \Ra\quad \int_{-1}^1 g(rs) (1-s^2)^{n-1}\ds \equiv 0
\]
for $r\geq 1$ and $n\geq 1$. Since $d$ is odd, we can reduce the integer exponent $n= \frac{d-3}2$ inductively until $n=0$. Then, by the same consideration as in the case $d=3$, the result is proved.
\end{proof}

We now prove the abstract statement about measures on the unit interval. 

\begin{proof}[Proof of Lemma \ref{lemma auxiliary bernstein}]
{\bf Lower bound.} Let $\mu$ be a finite signed measure satisfying the moment conditions
\[\showlabel\label{eq moment conditions}
\int_{0}^1 s\,\d\mu_s =1, \qquad \int_{0}^1 s^{2k}\,\d\mu_s = 0\quad\forall\ 0\leq k\leq n-1
\]
 Then
\[\showlabel\label{eq moments and norm}
\int_{0}^1 s\,\d\mu_s = \int_{0}^1 \left(s- \sum_{k=0}^n a_k s^{2k}\right)\,\d\mu_s \leq \left\|s- \sum_{k=0}^n a_k s^{2k}\right\|_{L^\infty(0,1)}\,|\mu|([0,1])
\]
by definition. Taking the infimum over the parameters $a_0,\dots, a_n$ on the right, we find that 
\[
1 = \int_{-1}^1 s\,\d\mu_s \leq \dist_{L^\infty(0,1)}\big(s\mapsto s, \:\mathrm{span}\{1,s^2, \dots, s^{2n}\}\big)\cdot \|\mu\|
\]
i.e.\
\[
\|\mu\| \geq \frac{1}{\dist_{L^\infty(0,1)}\big(s\mapsto s,\: \mathrm{span}\{1,s^2, \dots, s^{2m}\}\big)} = \frac1{\dist_{L^\infty(-1,1)}\big(s\mapsto |s|,\: \mathrm{span}\{1,s, \dots, s^{2m}\}\big)}.
\]
The asymptotics of 
\[
\beta_n:= \dist_{L^\infty(-1,1)}\big(s\mapsto |s|, \mathrm{span}\{1,s, \dots, s^{2m}\}\big)
\]
are known due to Bernstein \cite{bernstein1912ordre} and Varga and Carpenter \cite{varga1985bernstein} who proved that $\lim_{n\to \infty} n\,\beta_n =: \beta \approx 0.28$, so 
\[
\liminf_{n\to\infty}\frac{\gamma_n}n \geq \liminf_{n\to\infty} \frac1{n\beta_n} = \frac1\beta \approx 3.57.
\]

{\bf Upper bound: Step 0.} Note that due to compactness, there exist parameters $a_0, \dots, a_n$ such that
\[
\left\|s - \sum_{k=0}^na_ks^{2k}\right\|_{L^\infty(0,1)} = \dist_{L^\infty(0,1)}\big(s\mapsto s, \:\mathrm{span}\{1,s^2, \dots, s^{2m}\}\big).
\]
We fix $a_0, \dots, a_n$ accordingly. Further note that equality is attained in \eqref{eq moments and norm} if the measure $\mu$ is supported on the set of points
\[
\Theta := \left\{ s\in [0,1] : \left|s - \sum_{k=0}^n a_k s^{2k}\right| = \max_{r\in[0,1]}\left|r - \sum_{k=0}^n a_k r^{2k}\right| \right\}
\]
and the measure 
\[\showlabel\label{eq no cancellations}
\widetilde \mu = \left(s - \sum_{k=0}^n a_k s^{2k}\right)\cdot \mu
\]
which has density $s - \sum_{k=0}^n a_k s^{2k}$ with respect to $\mu$ is non-negative, i.e.\ $\mu$ has ``the right sign'' at all points. If such a $\mu$ exists, it therefore serves as a matching upper bound and the Lemma is proved. It is, however, not immediately clear whether there exists a signed measure $\mu$ supported on $\Theta$ which satisfies the moment conditions \eqref{eq moment conditions} and positivity condition \eqref{eq no cancellations}. In the following, we will prove that $\mu$ indeed does exist.

{\bf Step 1.} Due to compactness, $\Theta$ is a non-empty subset of $[0,1]$. Additionally
\[
\Theta \subseteq \{0,1\} \cup \left\{s \in \R : 2 \sum_{k=1}^nka_k\,s^{2k-1} = 1\right\}
\]
since the function $s\mapsto s - \sum_{k=0}^n a_k s^{2k}$ is either maximal or minimal at $s\in \Theta$. By the fundamental theorem of algebra, $\Theta = \{s_1, \dots, s_N\}$ is thus a finite subset of $[0,1]$. In this step, we prove that $0,1\in \Theta$ and $\Theta \cap (0,1) = n$.

Note that $\sum_{k=1}^n a_ks^{2k}$ is also an optimal polynomial approximation of the function $h(s) = |s|$ in $C^0[-1,1]$ in the space $\mathcal P_{2n+1}$ of polynomials of degree at most $2n+1$, since the optimal approximation is an even polynomial. By Chebyshev's equi-oscillation Theorem \cite[Section 6.9]{kincaid2009numerical}, there exist $N\geq 2n+3$ distinct points $t_1 < \dots < t_{N}$ such that the error 
\[
e(s) = |s|- \sum_{k=0}^n a_k s^{2k}
\]
satisfies 
\[
|e(t_i)| = \max_{s\in [-1,1]} |e(s)| \quad \forall\ i=1, \dots, N\qquad \text{and }e(t_i)e(t_{i+1}) <0 \quad\forall\ i=1, \dots, N-1,
\]
i.e. there are $N\geq 2n+3$ distinct points where the deviation from the target function is largest, and the oscillation around the target function at consecutive points $t_i, t_{i+1}$ goes in opposite directions.

Clearly, if $e$ is maximal at $s\in [-1,0]$ if and only if it is maximal at $(-s)\in [0,1]$. Therefore, there exist at least $\lceil N/2\rceil = \lceil(2n+3)/2\rceil = n+2$ points in $\Theta= [0,1]\cap \argmax e$. Rounding up is required since $2n+3$ is odd, and the point $0$ counts fully towards $\Theta\subset [0,1]$. Thus $|\Theta| \geq n+2$. 

It remains to show that $|\Theta|\leq n+2$. We prove this only if $n\geq1$, as the case $n=0$ of approximation by constant functions can be solved explicitly by direct inspection by the constant polynomial $a_0 = 1/2$.

Assume for a contradiction that $|\Theta| \geq n+3$. Then there exist at least $n+1$ distinct points in $\Theta\cap (0,1)$. Since $e$ is either maximal at $s\in \Theta\cap (0,1)$, we conclude that $e'(s) = 0$ for every $s\in \Theta \cap (0,1)$. By Rolle's Theorem, between any two points $s, s'$ such that $e'(s) = e'(s')$, there exists $s^*\in (s,s')$ such that $e''(s^*) = 0$. In particular, $e''$ has at least $n$ distinct zeros in $(0,1)$. Since $e''$ is even, it follows that $e''$ has at least $2n$ distinct zeros. But, since $e''$ is a polynomial of degree $2n-2$, it follows that $e''\equiv 0$ and thus that $e$ is a quadratic polynomial on $(0,1)$. On the other hand, we have seen that there exist at least $n+1$ points in $\Theta\cap (0,1)$, meaning that there are $n+1>1$ points in $(0,1)$ at which $e'$ vanishes. We conclude that $e'\equiv 0$, i.e.\ $e$ is a linear polynomial on $(0,1)$. It is easy to see that this is not optimal in terms of approximation.

{\bf Step 2.} We claim that the $(n+2)\times (n+2)$-Vandermonde type matrix
\[
V = \begin{pmatrix} s_0 & s_1& \dots &s_{n+1}\\1&1& \dots &1\\  s_0^2 & s_1^2 & \dots &s_{n+1}^2\\ s_0^4 & s_1^4 & \dots &s_{n+1}^4\\ \vdots &\vdots &\ddots &\vdots\\ s_0^{2n} & s_1^{2n} & \dots &s_{n+1}^{2n}\end{pmatrix}
\]
is invertible for any distinct $n+2$ points $0\leq s_0<\dots< s_{n+1}\leq 1$. This is true by classical results \cite{lundengaard2017generalized} for the $(n+1)\times(n+1)$ Vandermonde submatrix
\[
V = \begin{pmatrix} 1& \dots &1\\ s_0^2 & \dots &s_{n}^2\\ \vdots  &\ddots &\vdots\\ s_0^{2n} & \dots &s_{n}^{2n}\end{pmatrix}
\]
since the points $s_0^2, \dots, s_n^2$ are distinct.  It remains to show that the first row is linearly independent from the others, i.e.\ there exist no coefficients $a_0, \dots, a_k$ such that $s = \sum_{k=0}^n a_k s^{2k}$ at $(n+2)$ distinct points in $[0,1]$. Assume the contrary. Then there are $n+2$ distinct points $s_0 < \dots < s_{n+1}\in [0,1]$ such that 
\[
0 = s - \sum_{k=0}^{n} a_k s^{2k}, \qquad s\in \{s_0, \dots, s_{n+1}\}.
\]
By Rolle's theorem, between two such points $s_i, s_{i+1}$ there exists $\xi_i$ such that 
\[
0 = \frac{d}{ds}\bigg|_{s=\xi_i} \left(s - \sum_{k=0}^{n} a_k s^{2k}\right).
\]
The contradiction follows as in Step 1 of this proof.

{\bf Step 3.} Combining the results of the second and third step of this proof, we can choose $\{s_0, \dots, s_{n+1}\} = \Theta$ and find a unique vector $\nu\in \R^{d+2}$ such that 
\[\showlabel\label{eq linear moment system}
V\nu= (1, 0,0,\dots, 0)^T 
\]
The measure under consideration is now
\[
\mu= \sum_{i=0}^{n+1} \mu_i \delta_{s_i}\quad\text{which satisfies } \int_0^1s^{2k}\,\d\mu_s = \begin{cases} (V\nu)_1 &k=0\\ (V\nu)_{k+2} &k\geq 1\end{cases} = 0, \quad \int_0^1s^{2k}\,\d\mu_s = (V\nu)_2 = 1
\]
by construction. Thus the moment conditions are met. It remains to show that $\mu_i \cdot \left(s- \sum_{k=0}^n a_k s^{2k}\right)$ does not change sign in order to ensure that equality is attained in H\"older's inequality. Using Chebyshev's equi-oscillation theorem again, it suffices to show that $\mu_i$ and $\mu_{i+1}$ have opposite signs for all $i$.

For any $i\in\{0,\dots, n\}$, consider the unique even polynomial $P$ of degree $n$ such that $P(s_j) = 0$ for $0\leq j\leq n+1$ except $j\in \{i, i+1\}$. Then, since $P$ is an even polynomial of degree $\leq 2n$
\[\showlabel\label{eq consecutive oscillation}
0 = \int_0^1 P(s)\,\d\mu_s = \mu_i P(s_i) + \mu_{i+1}P(s_{i+1}),
\]
but since $P$ has $2n$ zeros at $\pm s_j$ for $i\notin\{i, i+1\}$, we find that $P(s)\neq 0$ for any $s\in[s_i, s_{i+1}]$. Thus $P(s_i)$ and $P(s_{i+1})$ have the same sign. In order to satisfy \eqref{eq consecutive oscillation}, we therefore find that $\mu_i$ and $\mu_{i+1}$ must have different signs.
\end{proof}

\begin{proof}[Proof of the claim in the proof of Theorem \ref{theorem main 3}]
To see this, we use the lower bound
\[\showlabel\label{eq moments and norm 2}
\|\mu\|\geq \frac1{\dist_{L^\infty(\eps,1)}\big(s,\:\mathrm{span}\{1,s^2,\dots,s^{2n}\}\big)}
\]
from \eqref{eq moments and norm}. By replacing the variable $s$ by $s^2$, we find that 
\[
\dist_{L^\infty(\eps,1)}\big(s,\:\mathrm{span}\{1,s^2,\dots,s^{2n}\}\big) = \dist_{L^\infty(\eps^2 ,1)}\big(\sqrt s,\:\mathrm{span}\{1,s,\dots,s^{n}\}\big).
\]
Recall that the function $\sqrt s$ is an analytic function on the interval $[\eps^2, 1]$ and
\[
\sqrt{s} = 1 + \sum_{n=1}^\infty(-1)^{n+1}\frac{\prod_{k=1}^n (2k-1)}{2^n\,n!}(s-1)^n
% = 1 - \sum_{n=1}^\infty \frac{\prod_{k=1}^n (2k-1)}{2^n\,n!}(1-s)^n
 = 1 - \frac1{\sqrt\pi} \sum_{n=1}^\infty \frac{\Gamma(n+1/2)}{\Gamma(n+1)} (1-s)^n.
\]
The coefficients decay asymptotically as $n^{-1/2}$ since
\[
\lim_{n\to\infty}\left(\sqrt n\, \frac{\Gamma(n+1/2)}{\Gamma(n+1)}\right) = 1,
\]
so for every $\delta>0$, there exists $N\in\N$ which is independent of $\eps$ such that the $L^\infty$-distance of the function $s\mapsto\sqrt s$ from the space $\mathcal P_n$ of polynomials of degree $\leq n$ is at most
\begin{align*}
\dist_{L^\infty(\eps^2 ,1)}\big(\sqrt s,\:\mathcal P_n \big) &\leq \max_{s\in[\eps^2, 1]} \left|\sqrt{s} - 1 - \frac1{\sqrt\pi} \sum_{k=1}^n \frac{\Gamma(k+1/2)}{\Gamma(k+1)} (1-s)^k\right|\\
	&\leq \frac{1+\delta}{\sqrt{\pi n}} \sum_{k=n+1}^\infty  (1- \eps^2)^k\\
	&= (1+\delta)\frac{(1-\eps^2)^{n+1}}{ \sqrt{\pi n} \,\eps^2}.
\end{align*}

\end{proof}

\section{Brief proofs of known results}\label{appendix known results}

In this appendix, we merely sketch the proofs of known results. For a more detailed introduction, we recommend e.g.\ \cite{barron_new}. We begin by sketching a proof of Proposition \ref{proposition properties of barron functions}, where we established general properties of Barron 

\begin{proof}[Proof of Proposition \ref{proposition properties of barron functions}]
{\bf First claim.} We note that, assuming existence of the integrals and for fixed $x\in\R^d$, we have
\begin{align*}
f_\pi(0) &= \int_{\R^{d+2}}a\,\sigma(b)\,\d\pi\\
\left|f_\pi(x) - f_\pi(0)\right| &= \left|\int_{\R^{d+2}}a\,\big\{\sigma(w^Tx+b)-\sigma(b)\big\}\,\d\pi\right|
	\leq \int_{\R^{d+2}}|a|\,|w^Tx|\,\d\pi
	\leq \frac{|x|}2\int_{\R^{d+2}}|a|^2 + |w|^2 \,\d\pi.
\end{align*}
If the first integral exists, then also the integral defining $f_\pi(x)$ exists as the integrand is continuous and grows at most linearly. Then
\[
f_\pi(x) = f_\pi(x) - f_\pi(0) + f_\pi(0) = \int_{\R^{d+2}}a\,\big\{\sigma(w^Tx+b)-\sigma(b)\big\}\,\d\pi + f_\pi(0).
\]
Measurability is not an issue for fixed $x$ due to the continuity of the integrand.
For the sake of brevity, denote $h_{(a,w,b)}(x) = a\,\big\{\sigma(w^Tx+b) -\sigma(b)\big\}$. More generally, we note that the Bochner-integral
\[
f = c + \int_{\R^{d+2}} \left(x\mapsto h_{(a,w,b)}(x)\right) \,\d\pi_{(a,w,b)}
\]
converges in $C^0(K)$ for compact sets $K$ and in $L^p(\P)$ for $1\leq p<\infty$ and probability distributions $\P$ with finite $p$-th moments in $x$, i.e.\ the function $(a,w,b)\mapsto h_{(a,w,b)}$ is Bochner integrable with respect to $\pi$ when considered as a function with values in either $C^0(K)$ or $L^p(\P)$. To see this, consider step functions
\[
\tilde h_i = \sum_j  1_{Q_{ij}} h_{(a_{ij},w_{ij},b_{ij})}, \qquad \tilde f_i = \int_{\R^{d+2}} \tilde h_i \,\d\pi
\]
where $Q_{ij}$ are $(d+2)$-dimensional cubes of side length $2^{-i}$ whose union is $\bigcup_j Q_{ij}= [-2^i, 2^i]^{d+2}$ and $(a_{ij},w_{ij}, b_{ij})\in W_{ij}$. If $(a,w,b) \in Q_{ij}$, then 
\begin{align*}
|h_{(a,w,b)}(x) - h_{(a_{ij}, w_{ij}, b_{ij})}(x)| &\leq |a-a_{ij}| \,|\sigma(w^Tx+b)| + |a_{ij}| \,\big|\sigma(w^Tx+b) - \sigma(w_{ij}^Tx+b_{ij}\big|\\
	&\leq |a-a_{ij}|\,|w^Tx+b| + |a_{ij}| \big[|w-w_{ij}|\,|x| + |b-b_{ij}|\big]\\
	&\leq C \left( \frac{|a-a_{ij}|^2 + |w-w_{ij}|^2 + |b-b_{ij}|^2}{\eps} + \eps\,\big\{|a|^2+ |w|^2|x|^2 + |b|^2 \big\}\right)
\end{align*}
for any $\eps>0$. Fixing $\eps$ to be the square root of the side-length of $Q_{ij}$, we find that $\tilde f_i(x)\to f_\pi(x)$ pointwise for all $x$. Furthermore, $\tilde f_i$ is Lipschitz continuous in $x$ uniformly in $i$, so $\tilde f_i$ converges to a limit in $C^0(K)$ by the compact embedding of Lipschitz functions in $C^0$, which coincides with the pointwise limit $f_\pi$. In other words, the Bochner integral exists in $C^0$. The argument follows in $L^p(\P)$ by the dominated convergence theorem considering $|\tilde f_i|(x) \leq 2(1 + [f]_\B|x|)$ for all $x\in \R^d$.

{\bf Second claim.} In this step, we show that $V_0$ is a Banach space and illustrate that $\B$ and $\B_0$ are different spaces. The fact that $V_0$ is a Banach space follows as \cite[Lemma 1]{siegel2021characterization} from the previous claim, where we have shown the existence of $f\in \B_0$ as a Bochner integral in $L^2(\P)$, i.e.\ as a continuous convex combination not only pointwise, but in a function space.

To see that $\B\neq \B_0$, observe that any $f\in \B$ can be decomposed into a positively one-homogeneous and a bounded part due to \cite[Corollary 5.3]{barron_new}. On the other hand, in one dimension, the function $f(x) = \log(1+x^2)$ satisfies $f(0) = f'(0) = 0$ and has an integrable second derivative $f''(x) = 2\frac{1-x^2}{(1+x^2)^2}$. By Proposition \ref{proposition one-dimensional}, we find that $f\in \B_0$. Since $f$ is not bounded but grows sub-linearly, we conclude that $\B\subseteq \B_0\not\subseteq \B$. The first inclusion follows from the fact that $[f]_\B\leq \|f\|_\B$ as shown next. 

{\bf Third claim.} The claim that $[f]_\B\leq \|f\|_\B$ is self-evident by definition, as the full Barron norm also limits the magnitude of the bias.

{\bf Fourth claim.} Finally, we note that $f\in \B_0$ is Lipschitz-continuous, since
\begin{align*}
|f_\pi(x) - f_\pi(x')| &= \left|\int_{\R^{d+2}} a\,\big[\sigma(w^Tx+b) - \sigma(w^Tx'+b)\big] \,\d\pi_{(a,w,b)}\right|\\
	&\leq \int_{\R^{d+2}} |a| \,|w^T(x-x')|\,\d\pi_{(a,w,b)} \leq |x-x'|\int_{\R^{d+2}} |a| \,|w|\,\d\pi_{(a,w,b)}. 
\end{align*}
Taking the infimum over $\pi$ (and optionally noting that $2|a|\,|w|\leq |a|^2+|w|^2$), we find that $|f(x) - f(x')| \leq [f]_\B|x-x'|$.
\end{proof}

Proposition \ref{proposition structure theorem} is proved in \cite[Theorem 5.18]{barron_new} and Corollary \ref{corollary structure barron} follows from it directly. Let us sketch how the structure of one-dimensional Barron functions can be understood.

\begin{proof}[Proof of Proposition \ref{proposition one-dimensional}]
{\bf Upper bound.} Let $a\in\R$ and $f\in C^2(\R)$ be such that $f''\in L^1(\R)$. Then for $x>a$ we have
\begin{align*}
f(x) &= f(a) + \int_a^xf'(s)\cdot 1\ds = f(a) + f'(a)\,(x-a)- \int_a^x f''(s)\,(s-x)\ds \\
	&= f(a) + f'(a)\,\sigma(x-a)+ \int_a^\infty f''(s)\,\sigma(x-s)\ds
\end{align*}
and for $x<a$
\begin{align*}
f(x) &= f(a) - \int_x^a f'(s)\cdot 1\ds = f(a) - f'(a)\,(a-x)- \int_a^x f''(s)\,(s-x)\ds \\
	&= f(a) - f'(a)\,\sigma(a-x)+ \int_a^\infty f''(s)\,\sigma(x-s)\ds.
\end{align*}
Noting that the $\sigma$ terms in the first expression vanish for when $x<a$ and vice versa, we find that
\[
f(x) = f(a) + f'(a)\big[\sigma(x-a) - \sigma(a-x)\big] + \int_\R f''(s) \big[ \sigma(x-s)\,1_{(a,\infty)}(s) + 1_{(-\infty,a)}(s)\,\sigma(s-x)\big]\ds.
\]
Consequently, $f= f_\mu$ for a measure
\[
\mu = f'(a) \,\big[ \delta_{(1,a)} - \delta_{(-1,a)}\big] + f''(b)\,\cdot\big[\H^1|_{\{w=1, b>a\}} + \H^1|_{\{w=-1,b<a\}}\big]
\]
where $\delta$ denotes the atomic point measure of mass one and $\H^1$ denotes the one-dimensional Hausdorff measure, restricted to half-lines $\{w=1, b>a\}$ and $\{w=-1,b<a\}$. and hence
\[
[f]_\B = \inf_{f=f_\mu} \|\mu\|_{TV} \leq 2\inf_{a\in\R} |f'(a)| + \int_\R |f''(s)|\ds.
\]
By approximation, the same is true if $f\notin C^2$ and $f''$ is merely a measure.

{\bf Lower bound direction.} The bound
\[
[f]_\B \leq [f]_{Lip} = \sup_{a\in\R} \max_{v\in \partial f(a)} |v| = \sup_{a\in\R}|f'(a)|
\]
follows from Proposition \ref{proposition properties of barron functions} and the Rademacher Theorem on the differentiability of Lipschitz functions.
For the second form of the lower bound, let $f\in \B_0$, i.e.\ there exists a measure $\mu$ on $\R^2$ such that 
\begin{align*}
f(x) &= \int_{\R^2} \sigma(w^Tx+b) \d\mu_{(w,b)} =\int_{\{w=0\}} \sigma(b)\d\mu_{(w,b)} +  \int_{\R^2}|w|\,\sigma\big(w/|w|\,x + b/|w|\big) \,\d\mu_{(w,b)}.
\end{align*}
The second expression can be written as 
\[
f(x) = c +  f^+(x) + f^-(x) = c + \int_{\R} \sigma(x+b) \,\d\mu^+_{b} + \int_{\R} \sigma(-x+b) \,\d\mu^+_{b}
\]
where 
\[
\mu_\pm = \psi_\sharp \big(|w| \cdot 1_{\{\pm w>0\}}\cdot\mu\big), \qquad \psi(w,b) = b/|w|,
\]
i.e.\ $\mu_\pm$ is the push-forward of the measure which has density $|w|$ with respect to $\mu$ onto the real line.
If $\phi \in C_c^\infty(\R)$ is any function, then by exchanging the order of integration and integrating by parts, we find that
\begin{align*}
\int_{-\infty}^\infty f^+(x)\,\phi''(x)\dx &= \int_{-\infty}^\infty \phi''(x)\int_{\R} \sigma(x+b) \,\d\mu^+_{b} \dx\\
	&= \int_{\R} \int_{-b}^\infty \phi''(x)\,(x+b) \dx\,\d\mu^+_{b} \\
	&= - \int_{\R} \int_{-b}^\infty \phi'(x) \dx\,\d\mu^+_{b}\\
	&= \int_\R \phi(b) \,\d\mu^+_b
\end{align*}
we find that $(f^+)''=\mu^+$ in the distributional sense, and thus $f'' = \mu^+ + \mu^-$. In particular, 
\[
\|f''\|_{TV} = \|\mu^++\mu^-\|_{TV} \leq \inf_\mu \|\mu^+\|_{TV} + \|\mu^-\|_{TV} \leq \inf_\mu \int_{\R^2} |w|\,\d|\mu|_{(w,b)} = [f]_\B.
\]
\end{proof}

We sketch a proof of the direct approximation theorem for Barron spaces.

\begin{proof}[Proof of Proposition \ref{proposition direct approximation}]
{\bf Step 1.} Consider the Hilbert space $L^2(\P)$ and observe that $h_{(a,w,b)}\in H$ defined by $h_{(a,w,b)}(x) = a\,\big\{\sigma(w^Tx+b)-\sigma(b)\}$ has norm at most
\[
\|h_{(a,w,b)}\|^2_H = \int_{\R^d} a^2\big[\sigma(w^Tx+b)-\sigma(b)\big]^2\,\d\P \leq a^2 \int_{\R^d}|w^Tx|^2\,\d\P.
\]
We use Proposition \ref{proposition properties of barron functions} to write $f\in\B_0$ as
\[
f(x) = f(0) + \int_{\R^{d+2}} h_{(a,w,b)}(x)\,\d\pi_{(a,w,b)}.
\]

{\bf Step 2.} Using the homogeneity relation $\sigma(z) = \lambda^{-1}\sigma(\lambda z)$, the distribution $\pi$ can be normalized such that 
\[
|a|^2 = |w|^2 = \frac12\int_{\R^{d+2}}|a'|^2 + |w'|^2 \,\d\pi_{(a',w',b')}
\]
almost surely by considering the push-forward of $\pi$ along the map 
\[
T: \R^{d+2} \to \R^{d+2}, \qquad T(a,w,b) = \left(a\,\sqrt{\frac{|w|}{|a|}}, \: w \,\sqrt{\frac{|a|}{|w|}},w \,\sqrt{\frac{|a|}{|w|}} \: \right)
\]
if $a, w\neq 0$ and $T(a,w,b) =0$ otherwise, which satisfies $f_{T_\sharp\pi}\equiv f_\pi$. Thus for any $\eps>0$, $f-f(0)$ is in the $H$-closed convex hull of the family 
\[
\mathcal G_{\|f\|_\B+\eps} = \{h_{(a,w,b)} : |a| = |w| \leq \|f\|_\B + \eps\}.
\]

{\bf Step 3.} By the Maurey-Barron-Jones Lemma \cite[Lemma 1]{barron1993universal}, for every $m\in\N$ and every $\eps'>0$, there exist $h_{(a_i,w_i,b_i)} \in \G_{\|f\|_\B + \eps}$ such that 
\[
\left\|f - f(0)- \frac1m \sum_{i=1}^m h_{(a_i, w_i, b_i)}\right\|_H \leq \frac{\|f\|_\B+\eps}{\sqrt m} + \eps'.
\]
As the vectors $(a_i,w_i,b_i)$ are constrained to a compact domain of $\R^{d+2}$ and the map $\R^{d+2}\to H$, $(a,w,b)\mapsto h_{(a,w,b)}$ is continuous, we can set $\eps, \eps'\to 0$ and obtain the result without constant by an appropriate subsequence. 

Finally, we write $c= f(0) + \frac1m \sum_{i=1}^m a_i \sigma(b_i)$ for compatibility with the original notation.
\end{proof}

\section{Further results}\label{appendix extensions}

\subsection{On the decay of $f_d^*(x)$ for $x\neq 0$} 
Numerical experiments in Appendix \ref{appendix plots} suggest that $f_d^*(x)$ decays to zero exponentially fast for $x\neq 0$. While we cannot prove this in full generality, we show that
\[
0 \leq f_d^*(x) \leq C\,d^{3/2}\,\left(\frac{1-|x|^2}{|x|}\right)^\frac{d-3}2
\]
for a constant $C>0$ which is independent of $d$.
In particular, $f_d^*(x) \to 0$ exponentially fast in $d$ if $|x|>0.62$. To see this, observe that
\begin{align*}
f_d^*(r) &= c_d \int_{-1}^1 g(rs)\,(1-s^2)^\frac{d-3}2\ds = 2c_d\,r^{\frac{1-d}2}\int_{0}^r g(z) (r^2-z^2)^\frac{d-3}2 \dz\\
	&= -2c_d\,r^{\frac{1-d}2} \int_{r}^1 g(z) (r^2-z^2)^\frac{d-3}2 \dz
\end{align*}
for $r<1$ since $g$ is $L^2(0,1)$-orthogonal to the polynomial $(r^2-z^2)^\frac{d-3}2$. Since $\|g\|_{L^\infty(0,1)}\leq \gamma_\frac{d-1}2$, we may estimate 
\[
|f_d^*(r)| \leq 2c_d \,r^\frac{1-d}2\,\gamma_d\,(1-r^2)^\frac{d-3}2 = \frac{2c_d\gamma_{\frac{d-1}2}}r\, \left(\frac{1-r^2}r\right)^\frac{d-3}2.
\]
The pre-factor grows as $d^{3/2}$ since $\gamma_d \sim d$ and 
\[
c_d = \frac1{\int_{-1}^1 (1-s^2)^\frac{d-3}2\ds} = \frac{\Gamma\left(d/2\right)}{\sqrt\pi\,\Gamma\left(\frac{d-1}2\right)} \sim \sqrt{\frac{d}{2\pi}}.
\]
Finally, we note that $(1-r^2)/r <1$ holds for positive $r$ if and only if $r> \frac{\sqrt 5 -1}2\approx 0.618$.

\subsection{Non-radial minimum norm solutions}

In this note, we constructed 
\[\showlabel\label{eq minimization appendix}
f_d^* \in \argmin_{f\in \F} [f]_\B, \qquad \F = \left\{f \in \B_0(\R^d) : f(0) =1 \text{ and } f\equiv 0 \text{ on }\R^d\setminus B_1(0)\right\}.
\]
Since both the Barron semi-norm and the class $\F$ are convex and invariant under rotations of the data domain, we find that there exists at least one minimizer which is radially symmetric. By direct construction, we saw that this minimizer 
\[
f_d^*(x) = 1 + \sum_{i=0}^\frac{d+1}2 \mu_i \avint_{S^{d-1}} \sigma(\nu^Tx-b_i)\,\d\H^{d-1}_\nu
\]
is unique, at least if $d$ is odd. The biases $0=b_0 < \dots < b_{(d+1)/2} = 1$ and weights $\mu_i \neq 0$ are given by the optimization process. Our proof does not exclude the existence of other minimizers, which are not radially symmetric. In fact, assume that $\phi_i\in L^\infty(S^{d-1})$ for $i=0,\dots, \frac{d+1}2$ are functions such that
\[\showlabel\label{eq structure fphi}
f_\phi(x) = \sum_{i=0}^\frac{d+1}2 \avint_{S^{d-1}} \sigma(\nu^Tx-b_i)\,\phi_i(\nu)\,\d\H^{d-1}_\nu = 0 \qquad\forall\ |x| \geq 1.
\]
Then trivially also $f_\phi(0) = 0$ since $b_i\geq 0$, and thus 
\[
(f_d^* + \eps f_\phi)(x) = 1 + \sum_{i=0}^\frac{d+1}2 \avint_{S^{d-1}} \big(\mu_i + \eps \phi_i(\nu)\big)\,\sigma(\nu^Tx-b_i)\,\d\H^{d-1}_\nu = \begin{cases} 1 &x=0\\ 0 &|x|\geq 1\end{cases}.
\]
Since $f_d^*$ is the {\em unique} radial solution, we can average in the radial direction and observe that $\int_{S^{d-1}}\phi_i(\nu) = 0$ for all $i=0,\dots,\frac{d+1}2$. The Barron norm of the combined solution is 
\[
\sum_{i=0}^{(d+1)/2} \frac{ \|\mu_i + \eps\phi_i\|_{L^1(S^{d-1})}}{\H^{d-1}(S^{d-1})} = \sum_{i=0}^{(d+1)/2} |\mu_i|
\]
if $\eps$ is so small that $\eps \|\phi_i\|_{L^\infty} \leq |\mu_i|$ for all $i$, since the function $\mu_i + \eps\phi_i$ does not change signs in this case, and the integral of $\phi_i$ averages to zero. In particular, if $(\phi_0,\dots, \phi_{(d+1)/2})$ exist such that $f_\phi$ is supported in $\overline{B_1(0)}$ and fails to be radial, then a non-radial minimizer exists. 

By considering the behavior of $f_\phi$ at infinity, we establish two conditions: $\sum_{i=0}^{(d+1)/2} \phi_i \equiv 0$ in order to have $f_\phi$ bounded, and $\sum_{i=0}^{(d+1)/2} b_i\phi_i\equiv 0$ in order to have $\lim_{x\to\infty}f_\phi(x) = 0$.

\begin{lemma}
Assume there exist $\frac{d+3}2$ measures $\bar \mu_i$ on $S^{d-1}$ such that
\[
f_{\bar\mu}(x) := \sum_{i=0}^\frac{d+1}2 \avint_{S^{d-1}} \sigma(\nu^Tx-b_i)\,\d\bar\mu_i = 0
\]
for all $|x| \geq 1$ and $f_{\bar\mu}(x)\not\equiv 0$. Then there exists a minimizer $\hat f_d\in \F$ of the Barron semi-norm which is not radially symmetric. Without loss of generality, we may assume that $\hat f_d$ is radially symmetric with respect to $(x_2,\dots,x_d)$.
\end{lemma}

\begin{proof}
{\bf Step 1.} Assume for now that $f_{\bar\mu}$ is identically zero. Let $\psi_\delta$ be a $C^\infty$-probability density on the group of rotations $SO(d)$ which is supported in an $\delta$-neighbourhood of the unit matrix, and let $H$ be the Haar measure on $SO(d)$. Define the radial mollification
\begin{align*}
f_{\bar\mu,\delta}(x) &= \int_{SO(d)} \psi_\delta(O) \,f_{\bar\mu}(O^{T}x)\,\d H_O\\
	&= \sum_{i=0}^\frac{d+1}2 \avint_{S^{d-1}} \left(\int_{SO(d)} \psi_\delta(O)\, \sigma((O\nu)^Tx-b_i) \,\d H_O
\right)\,\d\bar\mu_{i,\nu}\\
	&= \sum_{i=0}^\frac{d+1}2 \int_{S^{d-1}} \sigma(\nu^Tx-b_i)\,\d\tilde\mu_i
\end{align*}
where 
\[
\tilde\mu_{i,\delta}(B) = \int_{SO(d)} \psi_\delta(O)\,\bar\mu_i(O\cdot B)\,\d H_O.
\]
We make three observations.
\begin{enumerate}
\item $f_{\bar\mu,\delta}(x) =0$ if $x=0$ or $|x|\geq 1$. 
\item $f_{\bar\mu,\delta}\to f_{\bar\mu}$ as $\delta\to 0$ (pointwise and locally uniformly), so $f_{\bar\mu,\delta}$ cannot be identically zero for sufficiently small $\delta>0$.
\item $\tilde\mu_i$ is absolutely continuous with respect to the uniform distribution on the sphere since 
\[
|\tilde\mu_i| (B) \leq \|\psi_\delta\|_{L^\infty} \|\bar\mu_i\|_{TV}.
\]
Due to the uniform estimate, the Radon-Nikodym derivative $\phi_{i,\delta}:= \frac{\d \tilde\mu_{i,\delta}}{\d\H^{d-1}}$ is an $L^\infty(S^{d-1})$-function. 
\end{enumerate}

We now fix $\eps,\delta$ small enough, write $\phi_i = \phi_{i,\delta}$ and note that $f_d^* + \eps f_{\phi}$ is also a solution to \eqref{eq minimization appendix}. In particular, $f_\phi$ cannot be radially symmetric since $f_d^*$ is the unique radially symmetric minimizer.

{\bf Step 2.} Take $f_\phi$ to be non-trivial as implied by step 1. Then there exists at least one direction $\bar \nu$ such that $f_\phi(t\bar \nu) \not \equiv 0$. Without loss of generality, we may take $\bar\nu = e_1$. We can now average over all rotations which leave $e_1$ fixed. The resulting function $\hat f_\phi$ is radially symmetric in all components orthogonal to $e_1$, i.e.\ in $(x_2,\dots, x_d)$. Since we only average over rotations which leave the $e_1$-direction fixed, we have $\hat f_\phi(te_1)= f_\phi(te_1)\not\equiv0$. In particular, we may assume that $f_\phi$ has the desired symmetry.
\end{proof}

The question whether there exists $\bar\mu = (\bar\mu_0, \dots, \bar\mu_{(d+1)/2})$ such that $f_{\bar\mu} \equiv 0$ on $\R^d\setminus B_1(0)$ but $f_{\bar\mu}\not\equiv 0$ on $\R^d$ can be rephrased in terms of functional analysis. Namely, if we understand $\bar\mu$ as an element of the dual space $Z^*$ of $Z:= C^0(S^{d-1};\R^{(d+3)/2})$ and we associate to $x\in \R^d$ the function $h_x\in Z$ given by $\nu\mapsto \big(\sigma(\nu^Tx-b_0), \dots, \sigma(\nu^Tx-b_{(d+1)/2})\big)$, then we can write $f_{\bar\mu}(x) = \langle \bar\mu, h_x\rangle_{Z^*,Z}$ as a duality product.

In particular, we consider two subspaces $V_1, V_2\subseteq Z$:
\[\showlabel\label{eq v1 v2}
V_1 = \mathrm{span}\{h_x : x\in\R^d\}, \qquad V_2 = \mathrm{span}\{h_x : |x|\geq 1\}.
\]
Obviously $V_2\subseteq V_1$. We note the following: If $\overline{V_2}\neq \overline{V_1}$, then by the Hahn-Banach theorem there exists $\mu\in Z^*$ such that $\langle \mu, v\rangle = 0$ for all $v\in \overline{V_2}$ but not all $v\in \overline{V_1}$. It is easy to see by contradiction that there exists in particular $h_x$ with $|x|<1$ such that $f_\mu(x) = \langle \mu, h_x\rangle_{Z^*,Z} \neq 0$. Note that $x\neq 0$ since $f_\mu(0) = 0$ for any $\mu$ by design. 

We have thus proved the following.

\begin{corollary}
Denote $Z:= C^0(S^{d-1};\R^{(d+3)/2})$ and $h_x\in Z$, $h_x(\nu) = \big(\sigma(\nu^Tx-b_0), \dots, \sigma(\nu^Tx-b_{(d+1)/2})\big)$. Consider the subspaces $V_1, V_2$ of $Z$ as in \eqref{eq v1 v2}. There exists a non-radial solution $f$ of the minimization problem \eqref{eq minimization appendix} if and only if $\overline{V_1}\neq\overline{V_2}$.
\end{corollary}

\end{document}